\documentclass[sigconf]{acmart}
\setcopyright{none}

\usepackage{amsfonts}

\renewcommand\footnotetextcopyrightpermission[1]{} 
\usepackage{color}
\usepackage{algorithm}
\usepackage[noend]{algpseudocode}
\usepackage{bm}

\usepackage{subcaption}
\usepackage{graphicx}
\usepackage{lipsum}

\newtheorem{assumption}{\textbf{Assumption}}
\newtheorem{theorem}{\textbf{Theorem}}

\newtheorem*{theorem*}{Theorem}

\theoremstyle{definition}

\begin{document}

\title{ 
Neural Predictive Control to Coordinate Discrete- and Continuous-Time Models for Time-Series Analysis with Control-Theoretical Improvements
}

\author{Haoran Li, Muhao Guo, Yang Weng}
\affiliation{%
  \institution{Arizona State University}
  \city{Tempe}
  \state{AZ}
  \country{USA}
}
\email{{lhaoran, mguo26, yweng2}@asu.edu}

\author{Hanghang Tong}
\affiliation{%
  \institution{ University of Illinois at Urbana-Champaign}
  \city{Champaign}
  \state{Illinois}
  \country{USA}
}
\email{htong@illinois.edu}

\begin{abstract} 

Deep sequence models have achieved notable success in time-series analysis, such as interpolation and forecasting. Recent advances move beyond discrete-time architectures like Recurrent Neural Networks (RNNs) toward continuous-time formulations such as the family of Neural Ordinary Differential Equations (Neural ODEs). Generally, they have shown that capturing the underlying dynamics is beneficial for generic tasks like interpolation, extrapolation, and classification. However, existing methods approximate the dynamics using unconstrained neural networks, which struggle to adapt reliably under distributional shifts. In this paper, we recast time-series problems as the continuous ODE-based optimal control problem. Rather than learning dynamics solely from data, we optimize control actions that steer ODE trajectories toward task objectives, bringing control-theoretical performance guarantees. To achieve this goal, we need to $(1)$ design the appropriate control actions and $(2)$ apply effective optimal control algorithms. As the actions should contain rich context information, we propose to employ the discrete-time model to process past sequences and generate actions, leading to a coordinate model to \emph{extract long-term temporal features to modulate short-term continuous dynamics}. During training, we apply model predictive control to plan multi-step future trajectories, minimize a task-specific cost, and greedily select the optimal current action. We show that, under mild assumptions, this multi-horizon optimization leads to exponential convergence to infinite-horizon solutions, indicating that the coordinate model can gain robust and generalizable performance. Extensive experiments on diverse time-series datasets validate our method’s superior generalization and adaptability compared to state-of-the-art baselines.

\end{abstract}

\maketitle


\fancyfoot[R]{\scriptsize{Copyright \textcopyright\ 2021 by SIAM\\
Unauthorized reproduction of this article is prohibited}}


\fancyfoot[R]{\scriptsize{Copyright \textcopyright\ 20XX by SIAM\\
Unauthorized reproduction of this article is prohibited}}




\section{Introduction}
\label{sec:Intro}

Time-series analysis, such as classification, interpolation, and forecasting, has extensive applications in various domains like energy \cite{li2025external,li2025exarnn}, marketing \cite{sarkar2021lstm}, transportation \cite{nguyen2018deep}, etc. Recently, introducing dynamical modeling to this analysis has been widely recognized. For example, the continuous feature flow brings interpolative capacities at an arbitrary time, causing substantial improvements for scenarios like irregularly-sampled time-series data \cite{rubanova2019latent,kidger2020neural}. Additionally, the learned dynamics can naturally extrapolate future states, which is often more accurate than discrete updates in many smooth and continuous systems \cite{asikis2022neural,ref:Haoran2025L} or even in some stochastic systems like the PhysioNet dataset \cite{silva2012predicting,rubanova2019latent}. 

However, gaining accurate feature dynamics is challenging. There may exist a smooth dynamical function for many continuous systems, including engineering systems (e.g., power, gas, water, mechanical, and transportation systems) 
\citep{ref:Haoran2023P}, chemical reactions \citep{luis2010chemical}, biological networks ~\cite{palsson2011systems}, weather systems \citep{krishnamurthy2019predictability}, etc. However, for other time-series datasets, frequent disturbances or events could intermittently change dynamical evolution \citep{ref:Haoran2022T}, which calls for a highly adaptive dynamical function to different contexts. As contextual information can be extracted from long-term historical patterns, it's critical to merge both the adaptive short-term dynamics and long-term patterns innovatively.

To understand temporal patterns, discrete-time Deep Learning (DL) methods have been widely utilized, such as Recurrent Neural Networks (RNNs) \citep{chung2014empirical} or Long Short-Term Memory (LSTM) \citep{hochreiter1997long} and Transformer-based models \citep{zhou2021informer,wen2022transformers}. They process discrete observations sequentially but ignore the intra-dynamics between adjacent observations. Consequently, many studies demonstrate the model's insufficiency to tackle data from continuous systems  \citep{rubanova2019latent,kidger2020neural,morrill2021neural,chen2024contiformer,schirmer2022modeling}. Moreover, without dynamical knowledge, the performance of discrete-time DL degenerates significantly if the data is irregularly sampled \citep{rubanova2019latent}.

To incorporate dynamics, existing work can be categorized into the following groups. First, the dynamics between observations can be explicitly learned via methods such as Neural Ordinary Differential Equation (Neural ODE) \citep{chen2018neural}. However, Neural ODE depends on the initial value without adaptation capacities \citep{kidger2020neural}. The dynamics keep changing for complicated time series with disturbances and events. There needs to have a mechanism to adjust the dynamical evolution. Therefore, the second group increases the capacity of the model by adding control actions to the data dynamics. For example, Neural Controlled Differential Equation (Neural CDE) \citep{kidger2020neural} and Neural Rough Differential Equation (Neural RDE) \citep{morrill2021neural} create a rough path based on observed data to control the evolution of the dynamics. Another popular model is a learnable State Space Model (SSM) \citep{gu2021efficiently,smith2022simplified,schirmer2022modeling,ansari2023neural,gu2022parameterization}. SSM-based methods linearly combine feature states and control vectors. Generally speaking, these continuous-time models rely on observational data as control actions and overlook the high-level temporal patterns from the past \citep{chen2024contiformer}. The historical information of time series, however, is critical to providing contextual information and controlling dynamical propagation. 

There has been limited exploration into integrating discrete- and continuous-time models for time-series tasks. ODE-RNN \citep{rubanova2019latent} and GRU-ODE \citep{de2019gru} use discrete-time hidden states to modulate continuous dynamics, leveraging RNN architectures to inform ODE evolution. \citep{chen2024contiformer,jhin2024attentive} combine Neural ODEs or Neural CDEs with Transformer architectures \citep{vaswani2017attention}, using attention mechanisms to compute control actions. However, from a control-theoretic perspective, these approaches are inherently non-robust. They primarily focus on \emph{single-horizon} evaluations by optimizing the current action for immediate objectives. This neglects the long-term influence of actions on future feature trajectories. As a result, the induced dynamics may not converge to a stable equilibrium or a desired reference trajectory \citep{kouvaritakis2016model}.

In this paper, we cast the Neural ODE-based time-series learning problem as an optimal control problem \citep{NEURIPS2020_293835c2,ruiz2023neural}. Hence, we show that the model's performance and adaptation capacity are deeply linked to the convergence of the control problem. For example, (1) \textbf{Time-series classification}. The stable and high-level feature that corresponds to a label can be viewed as an equilibrium point of the feature ODE flow. The perturbations or distributional shift of the time series demand control strategies that always stabilize the ODE flow to the equilibrium. (2) \textbf{Time-series interpolation/extrapolation}. The reference feature flow can correspond to the true time-series dynamics for highly accurate interpolation and extrapolation. Effective control actions can enable converged approximation to these reference trajectories. Therefore, we propose a control-theoretical approach to design efficient \emph{actions} and \emph{control-based training algorithms}, providing certain convergence guarantees.

\begin{table}[h]
\centering
\caption{Table of Notation}
\label{tab:symbols}
\renewcommand{\arraystretch}{1.2}
\resizebox{\columnwidth}{!}{%
\begin{tabular}{ll}
\toprule
\multicolumn{2}{l}{\textbf{Scalars}} \\
\midrule
$t_i$               & Time for the $i^{th}$ observation \\
$N$        & Total number of observations for the time-series\\
$M$         & Number of horizons for the optimization in MPC \\
$y$          & Label of the time-series \\
$\lambda$          & Regularization penalty coefficient \\
\midrule
\multicolumn{2}{l}{\textbf{Vectors}} \\
\midrule
$\boldsymbol{x}_i$       & The $i^{th}$ observation of the time-series \\
$\boldsymbol{z}_i$       & The $i^{th}$ discrete hidden state in an RNN \\
$\boldsymbol{h}(t_i)$       & Continuous hidden state evaluated at time $t_i$ \\
$\boldsymbol{u}_i$ or $\boldsymbol{u}(t_i)$     & Action vector to control $\boldsymbol{h}(t)$ flow at time $t_i$\\
\midrule
\multicolumn{2}{l}{\textbf{Matrices}} \\
\midrule
$U_{i,M}$       & $M$-horizon control sequence \\
$H_{i,M}$       & $M$-horizon sequence of $\boldsymbol{h}(t)$\\
\midrule
\multicolumn{2}{l}{\textbf{Functions}} \\
\midrule
$g_{\psi}(\cdot)$        & Discrete-time updating function \\
$\ell^1_{\psi}(\cdot)$  & Neural network encoding discrete states to actions \\
$\ell_{\psi}^2(\cdot)$ &Neural network to map actions to label or data\\
$f_{\phi}(\cdot)$          & Neural network for the derivative $\frac{d\boldsymbol{h}(t)}{dt}$ \\
$\ell_{\phi}^1(\cdot)$  & Neural network encoding $\boldsymbol{x}_1$ to $\boldsymbol{h}(t_1)$ \\
$\ell_{\phi}^2(\cdot)$ & Neural network to map state to label or data \\
$X(t)$ & Interpolated continuous input path \\
$J(\cdot)$ & Task-dependent cost function \\
$\hat{J}(\cdot)$ & Regularization term for the actions\\
$\text{ODESolve}(\cdot)$ & Function to solve the initial value problem\\
$U_i(t)$ & Interpolated continuous control path from $t_i$ to $t_{i+M}$\\
\bottomrule
\end{tabular}%
}
\end{table}

Instead of utilizing current data or features as the control action, such as those in ODE-RNN or Neural CDE, we propose to extract the temporal features and forecast a sequence of actions for the future. Specifically, we employ an auto-regressive discrete-time model, e.g., an RNN, to understand the temporal information and forecast future actions. The auto-regression aims to provide rich contextual information for the current and future horizons. Hence, the output actions are highly expressive representations that can effectively adjust the continuous-time model, e.g., a Neural-ODE-based model. In general, the overall framework is a coordinate model, where the discrete auto-regressor generates sequential action features to control the evolution of the ODEs in the continuous-time model. The new architecture not only becomes highly capable of extracting and fusing long-term temporal features and short-term dynamics but also embraces control-theoretical training and analysis.

Specifically, we conduct training by solving a \emph{multi-horizon optimal control} problem, which chases convergence to the infinite-horizon solutions. A simple yet efficient approach is the so-called Model Predictive Control (MPC) \citep{kouvaritakis2016model,garcia1989model}. At each time, MPC solves a multi-horizon optimization to maximize the predictive performance and yield a sequence of optimal actions. Then, only the first control action is implemented. Such a process is repeated, bringing rigorous convergence guarantees \citep{veldman2022local}. Mathematically, in Theorems \ref{theo1} and \ref{theo2}, we introduce the \emph{exponential convergence} to a stable equilibrium or a reference trajectory. 

The overall model, dubbed Neural Predictive Control (NPC), leverages MPC during parameter updates. Specifically, with the produced control actions and the ODE dynamics, we construct a multi-horizon cost that can be iteratively minimized through gradient-based methods. In general, our NPC provides a unified framework for coordinating arbitrary pairs of discrete- and continuous-time models. It consistently outperforms each individual component, offers strong theoretical guarantees, and remains simple to implement. We validate its effectiveness on both synthetic and real-world time-series datasets, demonstrating substantial gains: we observe a $5\text{\%}\sim 15\text{\%}$ improvement in classification accuracy and a $30\text{\%}\sim 60\text{\%}$ reduction in mean squared error for regression tasks. Moreover, our framework remains highly scalable for high-volume and multi-dimensional time series, thanks to the highly efficient parallel computations in the discrete- \cite{chung2014empirical} and continuous-time models \cite{chen2018neural}.

\section{Problem Formulation}
\label{sec:back}


In this paper, we aim to solve the following problem.

\begin{itemize}
\item Goal: Build an accurate time-series classifier or regressor by incorporating continuous-time dynamics. 
\item Given: A series of observations $\{\boldsymbol{x}_i\}_{i=1}^N$ at times $\{t_i\}_{i=1}^N$ with potentially \emph{irregular intervals}. For a classification problem, the label $y$ of these observations is also available.
\item Find: A well-trained classifier or regressor to identify time-series labels or conduct accurate time-series interpolation or extrapolation. 
\end{itemize}

Then, we introduce some preliminaries.

\textbf{Discrete-Time Model}. We consider the sequence DL, such as RNNs, that contains a hidden state $\boldsymbol{z}$ to store temporal information. $\boldsymbol{z}$ can be updated at discrete time whenever a new observation is input to the model. Specifically, at time $t_i$, the updating function is:

\begin{equation}
\label{eqn:z_update}
\boldsymbol{z}_i=g_{\psi}(\boldsymbol{z}_{i-1},\boldsymbol{x}_i),
\end{equation}
where $g_{\psi}(\cdot)$ is a neural network with a parameter set $\psi$, e.g., a cell block in RNNs \citep{chung2014empirical}. The updated hidden state can be converted to the time-series label or forecast observations through an output layer. In our NPC framework in Section \ref{sec:method}, we will convert the state to a sequence of actions, leading to high-level temporal representations to capture current and future contextual information to adjust the continuous dynamics. 

\textbf{Remark}: There could be other candidates for the discrete-time model, such as Transformer-based models \cite{zhou2021informer}. In essence, a model is qualified as long as it can extract temporal information and produce sequential action representations.  

\begin{figure*}[h!]
\centering
\includegraphics[width=5.6in,trim=135 120 170 55, clip]{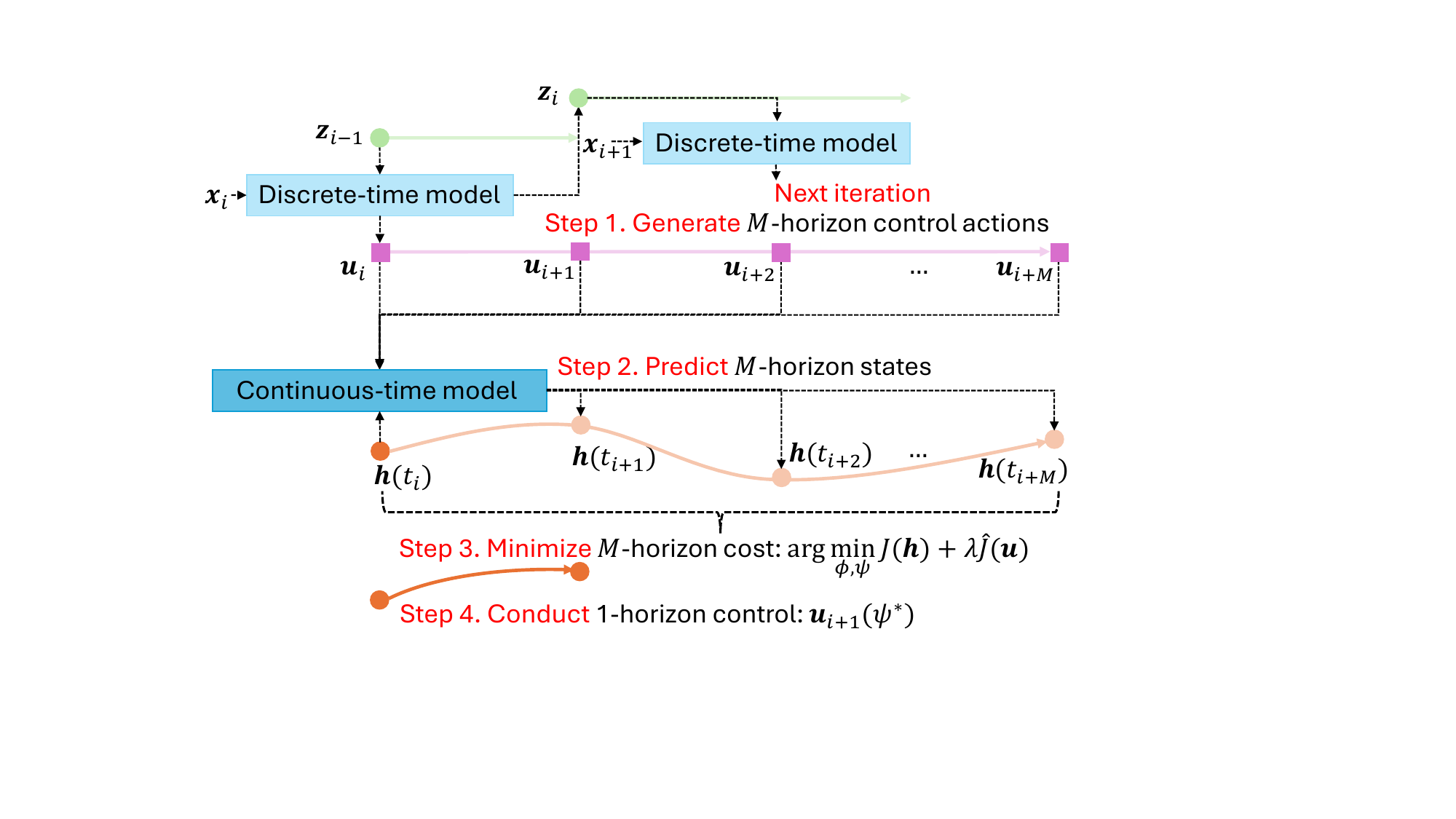}
\centering
\caption{The framework of the proposed NPC. Steps 1-4 demonstrate how to conduct $M$-horizon training at $t_i$.  }
\label{fig:framework}
\vspace{-4mm}
\end{figure*}

\textbf{Continuous-Time Model}. There are extensive Neural ODE variants that can model continuous-time dynamics. Their core component is the continuous feature ODE flow. Specifically, let $\boldsymbol{h}(t)$ denote the continuous feature states. A neural network $f_{\phi}(\cdot)$ is employed to depict the derivative of $\boldsymbol{h}(t)$:
\begin{equation}
\label{eqn:h_dynamic1}
\frac{d\boldsymbol{h}(t)}{dt}=f_{\phi}(\boldsymbol{h}(t),t),
\end{equation}
where $\phi$ is the parameter set of $f_{\phi}(\cdot)$. The derivative in Equation \eqref{eqn:h_dynamic1} can be adapted with respect to input time $t$. However, such a capacity is limited. Recent advances promote the adaptation capacity through introducing an additional context vector $\boldsymbol{u}(t)$.

\begin{equation}
\label{eqn:h_dynamic2}
\frac{d\boldsymbol{h}(t)}{dt}=f_{\phi}(\boldsymbol{h}(t),\boldsymbol{u}(t),t),
\end{equation}

For example, in Neural CDE \citep{kidger2020neural}, $\boldsymbol{u}(t)=\frac{dX}{dt}(t)$, where $X(t)$ is a continuous path created by the interpolation of the observations $\{\boldsymbol{x}_i\}_{i=1}^N$. Specifically, the interpolation uses a natural cubic spline with knots at $\{t_i\}_{i=1}^N$ such that $X(t_i)=(\boldsymbol{x}_i,t_i)$. In Augmented Neural ODE \cite{dupont2019augmented}, $\boldsymbol{u}(t)$ is an auxiliary augmented vector. 

\textbf{Hybrid Model}. In a hybrid framework, a discrete-time model $g_{\psi}(\cdot)$ is utilized to extract past information and provide more sophisticated context vectors. For example, in ODE-RNN \citep{rubanova2019latent}, $\forall t_i\leq t\leq t_{i+1}$, we have:

\begin{equation}
\label{eqn:h_dynamic3}
\frac{d\boldsymbol{h}(t)}{dt}=f_{\phi}\big(\boldsymbol{h}(t),g_{\psi}(\boldsymbol{z}_{i-1},\boldsymbol{x}_i),t\big), 
\end{equation}

In this formula, the context vector is defined as $\boldsymbol{u}_i=\boldsymbol{u}(t_i):=g_{\psi}(\boldsymbol{z}_{i-1},\boldsymbol{x}_i)$. Basically, the discrete state $\boldsymbol{z}_i$ in the RNN (see Equation \eqref{eqn:z_update}) works as the context vector to control $\boldsymbol{h}(t)$ flow from $t_i$ to $t_{i+1}$. Unlike the continuous control signal $\boldsymbol{u}(t)$ in Equation~\eqref{eqn:h_dynamic2}, the hybrid model employs a discrete context vector to improve efficiency. This design is justified, as the context vector can effectively capture local and relatively invariant environmental information. In Section \ref{sec:method}, we show that our framework can employ both discrete and continuous control patterns.

\textbf{Train the Hybrid Model by Solving the Optimal Control}. During training, $\boldsymbol{u}_i$ is also optimized through updating the parameter set $\psi$ of the RNN. Hence, we can view $\boldsymbol{u}_i$ as the control action, the RNN as the controller, and the learning problem as an optimal control problem. Specifically, we have:

\begin{equation}
\begin{aligned}
\label{eqn:opt_control1}
\arg\min_{\psi,\phi} &\ \ J(\boldsymbol{h}(t)),\\
\text{subject to} &\ \ \boldsymbol{h}(t_1)=\ell_{\phi}^1(\boldsymbol{x}_1),\\
&\ \ \boldsymbol{u}_i=g_{\psi}(\boldsymbol{z}_{i-1},\boldsymbol{x}_i),\\
& \ \ \boldsymbol{h}(t_{i+1})=\text{ODESolve}\bigg(f_{\phi}\big(\boldsymbol{h}(t_i),\boldsymbol{u}_i,t\big),\boldsymbol{h}(t_i),[t_i,t_{i+1}]\bigg)
\end{aligned}
\end{equation}
where $J(\cdot)$ is a cost function evaluated at the hidden trajectory of $\boldsymbol{h}(t)$. $J(\cdot)$ is task-dependent and can be written by the cross-entropy (CE) loss in the classification or the mean square error (MSE) in the regression, which will be fully explained in Section \ref{sec:method}. $\text{ODESolve}(\cdot)$, defined in \cite{chen2018neural}, solves the initial value problem in $[t_i,t_{i+1}]$, given the initial state $\boldsymbol{h}(t_i)$ and the derivative $f_{\phi}\big(\boldsymbol{h}(t_i),\boldsymbol{u}_i,t\big)$. $\ell_{\phi}^1(\cdot)$ is a neural network to encode the initial data $\boldsymbol{x}_1$ to an initial continuous hidden state.

However, the control is single-horizon. As shown in the equality constraint, the action $\boldsymbol{u}_i$ at $t_i$ will determine the flow of $\boldsymbol{h}(t)$ from $t_i$ to $t_{i+1}$. Nevertheless, the optimal action, under some optimal parameters $\psi^*$, can't ensure the robust performance for the subsequent flow after $t_i$. The problem is heavily discussed in dynamic programming, 
Reinforcement Learning (RL), and MPC \cite{busoniu2017reinforcement}. In this paper, we adopt MPC theory to renovate the training procedure in Equation \eqref{eqn:opt_control1} since MPC is suitable and practical for the optimal control problem. As a result, innovative model architecture and training algorithms are proposed, which completely changes the paradigms from existing Neural ODE-based time-series analysis.

\vspace{-4mm}
\section{Methods}
\label{sec:method}
In this paper, we propose a highly expressive and theoretically grounded model, Neural Predictive Control (NPC). The overall framework is shown in Fig. \ref{fig:framework}. Generally speaking, NPC assigns a discrete-time model (light blue box) to generate $M$-horizon future actions $[\boldsymbol{u}_i,\cdots,\boldsymbol{u}_{i+M}]$ (pink squares). These actions are input to a continuous-time model (dark blue box) to modify the evolution of a continuous hidden feature state $\boldsymbol{h}(t)$ (light brown curve). Then, following MPC's philosophy, an $M$-horizon optimization is solved and a $1$-horizon optimal action is conducted (dark brown curve). This approximates the solution of the infinite-horizon problem with a bounded and affordable error (see Theorems \ref{theo1} and \ref{theo2}). Steps $1\sim 4$ in Fig. \ref{fig:framework} summarize the process and should be repeated sequentially. Specifically, we elaborate on the procedure as follows.


\textbf{Discrete-time model to output predictive control sequence}. As shown in the top layer of Fig. \ref{fig:framework}, we utilize the following discrete-time model:
\begin{equation}
\label{eqn:controller}
\boldsymbol{z}_0=\boldsymbol{0},\boldsymbol{z}_{i}=g_{\psi}(\boldsymbol{z}_{i-1},\boldsymbol{x}_i), [\boldsymbol{u}_i,\boldsymbol{u}_{i+1},\cdots,\boldsymbol{u}_{i+M}]=\ell^1_{\psi}(\boldsymbol{z}_i),
\end{equation}
where $\ell^1_{\psi}(\cdot)$ encodes the discrete state to $M$-horizon control sequence $U_{i,M}:=[\boldsymbol{u}_i,\boldsymbol{u}_{i+1},\cdots,\boldsymbol{u}_{i+M}]$. These predictive actions determine the propagation of hidden states $\boldsymbol{h}(t)$ in the continuous-time model.

\textbf{The Proposed Predictive Optimal Control Framework}. Based on the predictive action sequence, we propose a new optimal control paradigm. Mathematically, let $H_{i,M}:=[\boldsymbol{h}(t_{i+1}),\cdots,\boldsymbol{h}(t_{i+M})]$, we introduce the following iterative $M$-horizon subproblem:

\begin{equation}
\begin{aligned}
\label{eqn:opt_control2}
\arg\min_{\psi,\phi} &\ \ {\color{blue}J(H_{i,M})+ \lambda \hat{J}(U_{i,M})},\\
\text{subject to} &\ \ \boldsymbol{h}(t_1)=\ell_{\phi}^1(\boldsymbol{x}_1),\\
&{\color{blue}\ \ U_{i,M}=\ell^1_{\psi}(\boldsymbol{z}_i),}\\
{\color{blue}H_{i,M}}&=\text{ODESolve}\bigg(f_{\phi}\big(\boldsymbol{h}(t_i),{\color{blue}\boldsymbol{U}_{i,M}},t\big),\boldsymbol{h}(t_i),{\color{blue}[t_i,\cdots,t_{i+M}]}\bigg)
\end{aligned}
\end{equation}

Compared to Equation \eqref{eqn:opt_control1}, important changes, induced by the predictive control strategy, are marked in {\color{blue}blue}. $\lambda>0$ is a penalty constant and $\hat{J}(\cdot)$ is a regularization term for the actions. The detailed formula for $J(\cdot)$ and $\hat{J}(\cdot)$ are presented in Equations \eqref{eqn:loss_class} to \eqref{eqn:reg_action}. In this optimization, starting from an initial state $\boldsymbol{h}(t_1)$, the action sequence controls the dynamics in the subsequent $M$ horizons, causing the feature sequence $H_{i,M}$ that are evaluated and optimized. The evolution computation differs according to different continuous ODE models. In this paper, we investigate two popular models: Neural CDE and ODE-RNN, which serve as excellent representatives of continuous-time and hybrid models, respectively.

\textbf{Predictive-sequence controlled Neural CDE}. The original Neural CDE \citep{kidger2020neural} is controlled by the raw data path $X(t)$. We generalize this idea and consider a \emph{control path} $U_i(t)$ ($1\leq i\leq N$) of $M+1$ horizons, where $U_{i}(t_{i+k})=(\boldsymbol{u}_{i+k},t_{i+k})$ is formulated based on $\{\boldsymbol{u}_{i+k},t_{i+k}\}_{k=0}^M$ obtained in Equation \eqref{eqn:controller}. Consequently, for $t_i\leq t\leq t_{i+M}$, our \emph{predictive Neural CDE} updates the continuous state $\boldsymbol{h}(t)$ with Riemann–Stieltjes integral \citep{mozyrska2009riemann}: 
\begin{equation}
\begin{aligned}
\label{eqn:ht_cde}
\boldsymbol{h}(t) &= \boldsymbol{h}(t_i)+\int_{t_i}^{t}f'_{\phi}(\boldsymbol{h}(s),s)\frac{dU_i}{ds}(s)ds,
\end{aligned}
\end{equation}
where the dynamics $f_{\phi}(\boldsymbol{h}(t),\boldsymbol{u}(t),t)=f'_{\phi}(\boldsymbol{h}(t),t)\frac{dU_i}{dt}(t)$, continuously changed by the control action $\frac{dU_i}{dt}(t)$. Then, $H_{i,M}$ can be inferred by inserting the time values $t_i$ to $t_{i+M}$. 

\textbf{Predictive-sequence controlled ODE-RNN}. In ODE-RNN \citep{rubanova2019latent}, the observations, processed by another RNN, work as a controller to sequentially change the initial state of $\boldsymbol{h}(t)$ for each interval $[t_i,t_{i+1}]$. To distinguish the built-in RNN from the discrete-time model $g_{\psi}(\cdot)$, we denote the cell block of the former as $g'_{\phi}(\cdot)$, and the subscript $\phi$ implies that $g'_{\phi}(\cdot)$ and $f_{\phi}(\cdot)$ have a common parameter set.  Hence, the following updates exist for each $t_i\leq t\leq t_{i+1}$:
\begin{equation}
\label{eqn:ht_ode_rnn}
\tilde{\boldsymbol{h}}(t) = \boldsymbol{h}(t_i) + \int_{t_i}^t f_{\phi}(\boldsymbol{h}(s),\boldsymbol{u}_{i},s)ds, 
\end{equation}
where control actions $\boldsymbol{u}_i$ are calculated from Equation \eqref{eqn:controller}. $H_{i,M}$ can be calculated from $t_i$ to $t_{i+M}$. 

To summarize, Equations \eqref{eqn:ht_cde} and \eqref{eqn:ht_ode_rnn} merge discrete-time model $g_{\psi}(\cdot)$ and continuous-time model $f_{\phi}(\cdot)$ to construct a coordinated model. Fig. \ref{fig:framework} demonstrates how the two models (light/dark blue boxes) work together to output the state flow $\boldsymbol{h}(t)$ (light brown curve). Subsequently, the flow from $\boldsymbol{h}(t_i)$ to $\boldsymbol{h}(t_{i+M})$ should be evaluated with a cost function in Optimization \eqref{eqn:opt_control2}, which is the core of MPC to greedily maintain optimality. Thus, we present the cost function according to different tasks.



\textbf{Classification cost}. For a classification problem, we define the cost as the classification loss evaluated at the terminal state at $t_{i+M}$. Hence, Optimization \eqref{eqn:opt_control2} drives the whole feature flow towards more separable regions. Specifically, we have:
\begin{equation}
\label{eqn:loss_class}
J(H_{i,M}) =L_{\text{CE}}\Big(\ell_{\phi}^2\big(\boldsymbol{h}(t_{i+M})\big),y\Big),
\end{equation}
where $\ell_{\phi}^2$ is a readout neural network to map from hidden states to labels, $y$ is the true label for input time series, and $L_{\text{CE}}(\cdot,\cdot)$ is the cross-entropy loss.

\textbf{Regression cost}. In a regression task, the cost is the mean square error (MSE, $L_{\text{MSE}}(\cdot,\cdot)$) for each state.

\begin{align}
\label{eqn:loss_reg}
J(H_{i,M}) &= \sum_{k=0}^M L_{\text{MSE}}\Big(\ell_{\phi}^2\big(\boldsymbol{h}(t_{i+k})\big),\boldsymbol{x}_{i+k}\Big),
\end{align}
where $\ell_{\phi}^2(\cdot)$, with a slight abuse of notation, maps from state to time-series values.

\textbf{Action regularization}. The norm of the action is used in MPC to mitigate control chattering and guarantee stability and convergence \citep{bejarano2023multi,veldman2022local}. However, this may limit the expressivity of the discrete-time model. Instead, we encourage the consistence of actions: they should all be tied to most relative features associated with the corresponding labels/values. To this end, we endow the discrete-time model with the capacity to solely complete the learning task:

\begin{align}
\label{eqn:reg_action}
\hat{J}(\psi) = 
\begin{cases}
\sum_{k=0}^M L_{\text{CE}}\Big(\ell_{\psi}^2\big(\boldsymbol{u}_{i+k}\big),y\Big) \\ 
\quad \quad \quad \quad \text{Classification loss}\\
\sum_{k=0}^M L_{\text{MSE}}\Big(\ell_{\psi}^2\big(\boldsymbol{u}_{i+k}\big),\boldsymbol{x}_{i+k}\Big) \\ 
\quad \quad \quad \quad \text{Regression loss},
\end{cases}
\end{align}
where $\ell_{\psi}^2(\cdot)$ is a neural network to convert actions to labels or time-series data.

\textbf{Training algorithms}.As shown in Algorithm \ref{alg:train_NPC_NPCF}, NPC solves Optimization \eqref{eqn:opt_control2} by gradient descent and only conducts $1$-horizon optimal control action to reach $\boldsymbol{h}(t_{i+1})$. This approach is known to provide good convergence to infinite-horizon minimization \citep{veldman2022local}. 

\begin{algorithm}[h]
\caption{MPC-based Training Algorithm to Train NPC.}\label{alg:train_NPC_NPCF}
\begin{algorithmic}
\State \textbf{Input:} A sequence of observations $\{\boldsymbol{x}_i\}_{i=1}^N$ at times $\{t_i\}_{i=1}^N$. For a classification problem, the time-series label $y$ is also provided. There can be multiple sequences and labels. 
\State \textbf{Initialize:} $\boldsymbol{z}_0=\boldsymbol{0}$. The number of look-ahead horizons $M$ and penalty term $\lambda$.
\While{Not converge}
\For{$i$ in $1,2,\cdots, N$}
\State Compute actions $U_{i,M}$ as outputs of $g_{\psi}(\cdot)$ in Equation \eqref{eqn:controller}.
\State Compute $H_{i,M}$ in Equation \eqref{eqn:ht_cde} or \eqref{eqn:ht_ode_rnn} with control actions and a continuous-time model $f_{\phi}(\boldsymbol{h}(t),\boldsymbol{u}(t),t)$. 
\State Conduct $M$-horizon minimization in Equation \eqref{eqn:opt_control2}. 
\State Conduct the action $\boldsymbol{u}_{i+1}(\psi^*)$, where $\psi^*$ is a current optimal solution. 
\EndFor
\EndWhile
\State \textbf{Output:} Optimal parameters $(\psi^*,\phi^*)$.
\end{algorithmic}
\end{algorithm}

\section{Theoretical Analysis}
\label{sec:theo}
In this section, we provide theoretical insights about how MPC can guide the training towards a more robust and generalized solution. In particular, we elaborate on the function class of $f_{\phi}(\boldsymbol{h}(t),\boldsymbol{u}(t),t)$ that can be linearized into a state-space model (SSM) with a bounded error, as shown in Assumption \ref{assum1} in Appendix \ref{appendix:assum}. Hence, linear ODE and control theorems can be used to analyze the infinite-horizon problem and derive the convergent distance between $M$-horizon and infinite-horizon solutions. This analysis sheds light on the general analysis, even when the derivative can be nonlinear: as shown in Equations \eqref{eqn:stability_error} and \eqref{eqn:generalizability}, the bounded linearization error can also converge. While it's hard to directly quantify the potential bounds as different time-series data have different complexities, a large amount of work \citep{gu2021efficiently,smith2022simplified,schirmer2022modeling,ansari2023neural,gu2022parameterization} reveals that SSM is highly competent for generic time-series modeling.

For a classification problem, without the loss of generality, the proposed NPC aims to start from $\boldsymbol{h}(t_1)$ and stabilize $\boldsymbol{h}(t)$ to the origin such that $\ell_{\phi}^2(\boldsymbol{0})\approx C\cdot\boldsymbol{0}=\boldsymbol{0}$, implying that the label $y=0$, where $C$ is defined in the SSM in Assumption \ref{assum1}. Other labels have similar processes with some variable transformations. For a regression problem, the control goal is to enable $\ell_{\phi}^2(\boldsymbol{h}(t))$ to certain values. 
Then, we answer the following two questions: ({\em Q1. Stability}). Will the flow of $\boldsymbol{h}(t)$ get stable at the origin with different initial values $\boldsymbol{h}(t_1)$? ({\em Q2. Generalizability}). Will the flow and control sequence under $M$-horizon optimization converge to the infinite-horizon optimal solutions? 

\begin{theorem}[Stability]
\label{theo1}
Assume Assumption \ref{assum1} holds. Let $T=t_{i+M}-t_i$, $\tau=t_{i+1}-t_i$ and $\boldsymbol{h}_{(\psi^*,\phi^*)}(t)$ denote the optimal state after NPC training in Algorithm \ref{alg:train_NPC_NPCF}. There exists constants $K$, $K_1$, $K_2$, $\mu_{\infty}$, and $M_{\infty}\geq 1$ such that 
\begin{align}
\label{eqn:stability_error}
||\boldsymbol{h}_{(\psi^*,\phi^*)}(t)||\leq  M_{\infty}e^{-\mu t}||\boldsymbol{h}_{(\psi^*,\phi^*)}^*(t_1)|| + \notag \\ \frac{1-e^{-\mu t}}{\mu}K(1+(L+1)\tau e^{K(L+1)\tau})||\boldsymbol{w}||_{l^{\infty}(0,t)},
\end{align}
where $\mu=\mu_{\infty}-K_1e^{-2\mu_{\infty}(T-\tau)}-K_2L-KL(L+1)\tau e^{K(L+1)\tau}$, $\boldsymbol{h}_{(\psi^*,\phi^*)(t)}^*$ is the globally optimal solution, and $||\cdot||_{l^{p}(0,t)}$ is the $l$-p norm on the function space over $(0,t)$.
\end{theorem}

The proof can be seen in Appendix \ref{appendix:proof1}. When the linearization error, measured by $L$ defined in Assumption \ref{assum1}, is small and $T-\tau$ is large (i.e., a large $M$ in the $M$-horizon optimization), $\mu>0$. Hence, the first term on the Right-Hand-Side (RHS) of Equation \eqref{eqn:stability_error} exponentially decreases as $t$ increases. The second term on RHS is limited when $\boldsymbol{w}(t)$ and $L$ are small. Thus, $\boldsymbol{h}_{(\psi^*,\phi^*)}(t)$ gets stable to the origin exponentially. For the generalizability, we have:

\begin{theorem}[Generalizability]
    \label{theo2} Assume Assumption \ref{assum1} holds. Consider $T$, $\tau$, $K_2$, $\mu_{\infty}$, and $\boldsymbol{h}_{(\psi^*,\phi^*)}(t)$ defined in Theorem \ref{theo1} and let $\boldsymbol{u}_{(\psi^*,\phi^*)}(t)$ denote the optimal control action after NPC training in Algorithm \ref{alg:train_NPC_NPCF}. Let $\boldsymbol{u}^*_{\infty}(t)$ denote the optimal solution of applying the linear model in Assumption \ref{assum1} to the infinite-horizon minimization problem, defined in Equation (3)
    in Appendix A 
    . Let $\boldsymbol{h}^*_{\infty}(t)$ denote the state controlled by $\boldsymbol{u}^*_{\infty}(t)$ using the non-linear model $f_{\phi^*}(\cdot)$. There exists a constant $K_3$ such that:
    \begin{equation}
    \begin{aligned}
    \label{eqn:generalizability}
||\boldsymbol{h}_{(\psi^*,\phi^*)}(t) - \boldsymbol{h}^*_{\infty}(t)|| + ||\boldsymbol{u}_{(\psi^*,\phi^*)}(t) - \boldsymbol{u}^*_{\infty}(t)|| \\ \leq K_3e^{-2\mu_{\infty}(T-\tau)}\Big(\frac{L+1}{\mu_{\infty}-K_2L}||\boldsymbol{h}||_{l^1(0,t)} 
+||\boldsymbol{h}(t)||\Big) 
\\ + K_3\tau e^{K_3(L+1)\tau}\frac{L+1}{\mu_{\infty}-K_2L}(L||\boldsymbol{h}||_{l^1(0,t)}+||\boldsymbol{w}||_{l^{\infty}(0,t)}).
\end{aligned}
    \end{equation}
\end{theorem}

The proof is in Appendix \ref{appendix:proof2}. By Theorem \ref{theo1}, $||\boldsymbol{h}(t)||$ has an exponential decay to $0$. Moreover, $T-\tau$ is sufficiently large. Thus, the first term of RHS in Equation \eqref{eqn:generalizability} decreases exponentially, and the second term is small as long as $\boldsymbol{w}(t)$ is sufficiently small, demonstrating the high generalizability.

\section{Experiments}

\label{sec:exp}
\subsection{Settings}
\label{sec:setting}

\textbf{Datasets.} We use the following datasets for experiments. (1) \textbf{Synthetic dataset}. We create a toy example to validate the stability of our NPC framework. The specific data generation process is described in Appendix \ref{sec:synthe_data}. Fig. \ref{fig:syn_train} and \ref{fig:syn_test} visualize the training and the test time series, and different colors (blue and brown) represent different labels. In the test dataset, we introduce different levels of deviations (i.e., brown colors from light to dark) to evaluate the stability. (2) \textbf{Human Activity Recognition (HAR) dataset.} HAR data \citep{anguita2013public} has recordings of $30$ subjects from waist-mounted smartphone with embedded inertial sensors. Observations include linear acceleration and angular velocity, and there are $6$ different types of labels describing different activities. $(3)$ \textbf{UCR Time Series Archive.} The archive \citep{UCRArchive} contains $85$ different types of time series from diverse domains. The time-series length ranges from $60$ to $2700$, and the number of label classes ranges from $2$ to $60$. We randomly select $9$ datasets to test. $(4)$ \textbf{Photovoltaic (PV) datasets.} We introduce a publicly available Photovoltaic (PV) dataset \citep{boyd2016nist} about the sequential solar power generations. The values are largely determined by the continuous movement of the sun and the wind. These datasets are selected due to diversified applications, complex and continuous dynamics, and potentially irregular samples. 


\begin{figure}[h]
    \centering
    \begin{subfigure}{0.5\textwidth}
        \centering
        \includegraphics[width=\textwidth]{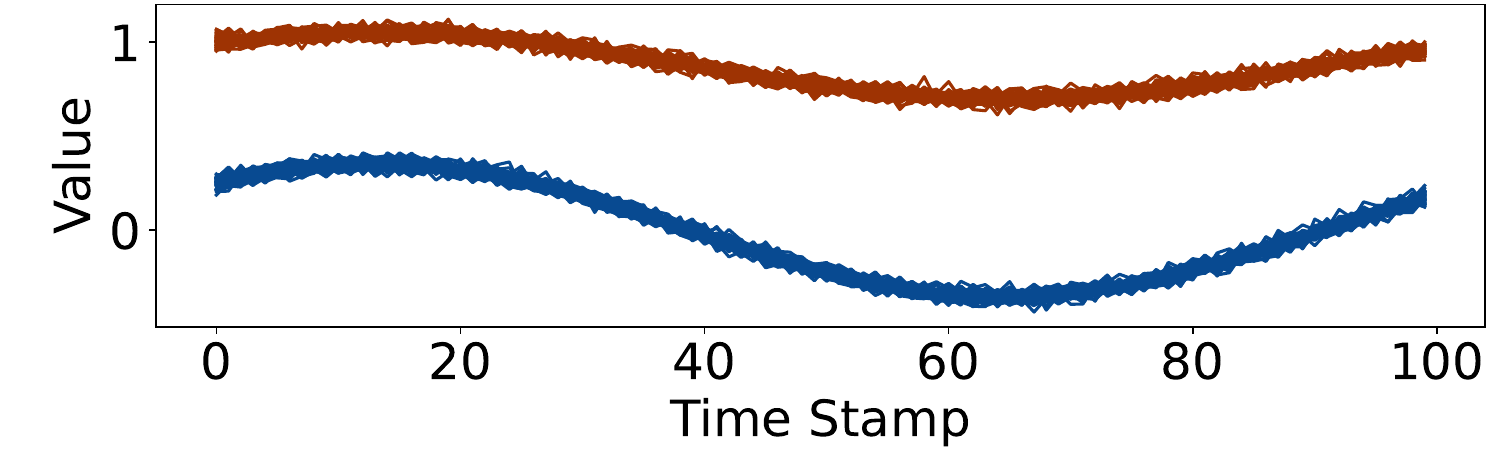}
        \caption{Train data of the toy example.}
        \label{fig:syn_train}
    \end{subfigure}
    \hspace{-2.5mm}
    \begin{subfigure}{0.5\textwidth}
        \centering
        \includegraphics[width=\textwidth]{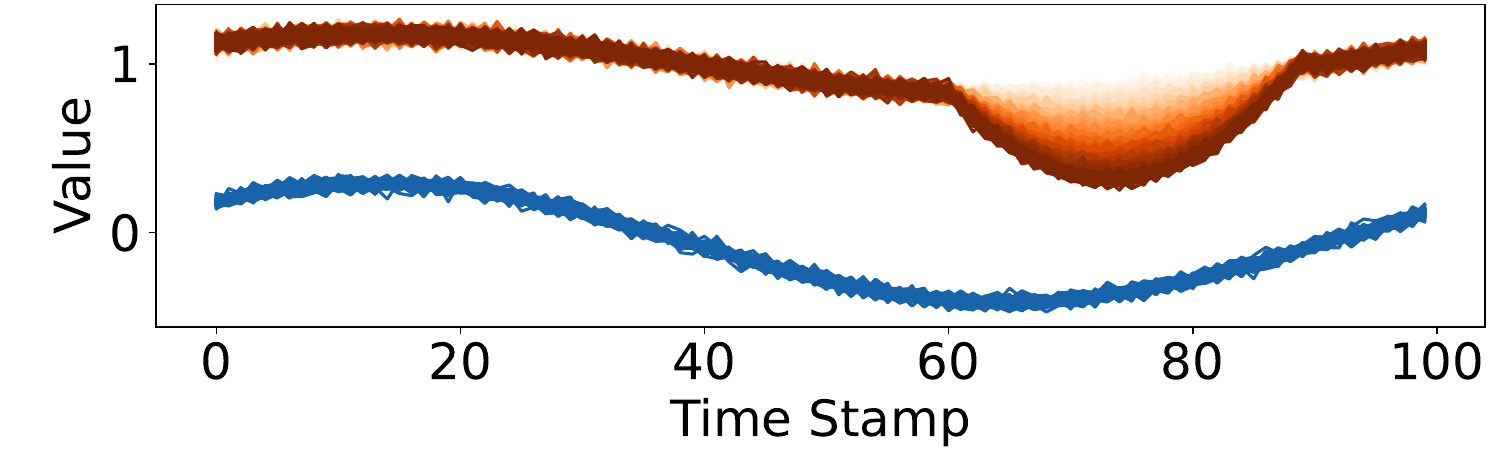}
        \caption{Test data of the toy example with deviations.}
        \label{fig:syn_test}
    \end{subfigure}
    \vskip\baselineskip
    \begin{subfigure}{0.5\textwidth}
        \centering
        \includegraphics[width=\textwidth]{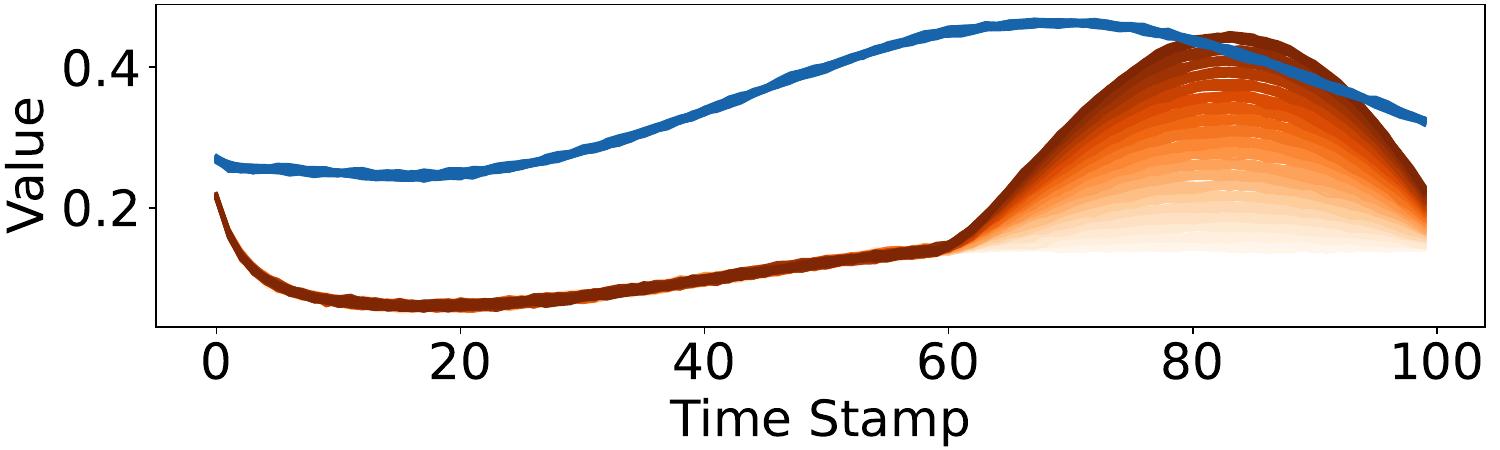}
        \caption{The evolution of $\boldsymbol{h}(t)$ in ODE-RNN for test sets. }
        \label{fig:ht_odernn}
    \end{subfigure}
        \hspace{-2mm}
    \begin{subfigure}{0.5\textwidth}
        \centering
        \includegraphics[width=\textwidth]{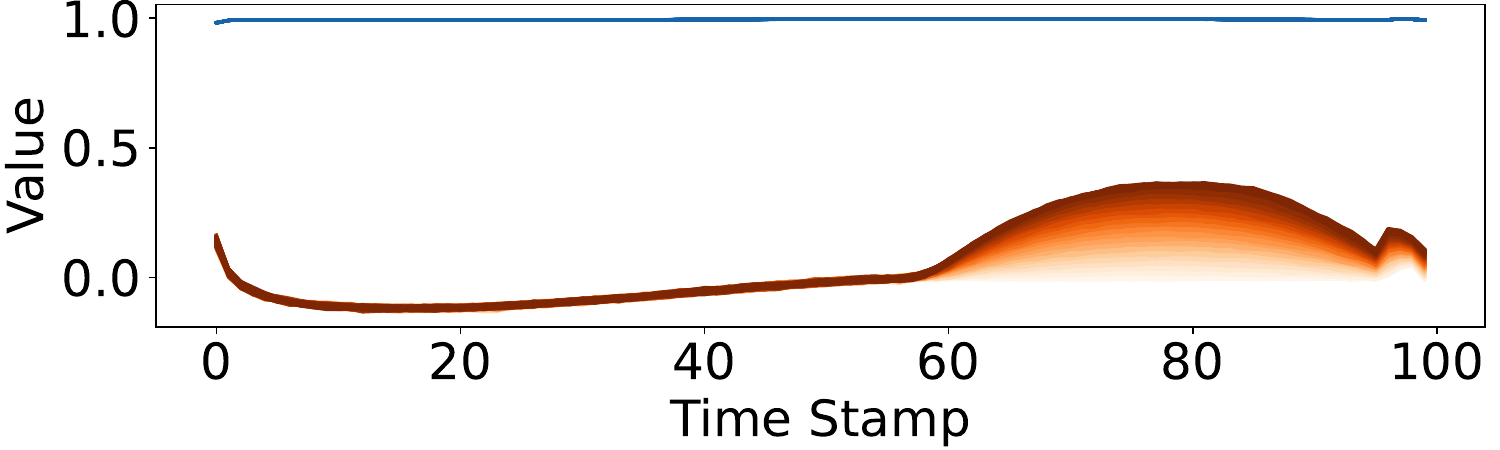}
        \caption{The evolution of $\boldsymbol{h}(t)$ in the proposed NPC for test sets.}
        \label{fig:ht_npc}
    \end{subfigure}
    \caption{Visualization of data and features in toy examples.}
    \label{fig:result_toy}
    \vspace{-4mm}
\end{figure}

\textbf{Benchmark methods}. The following methods are utilized as benchmarks. (1) RNN-$\Delta_t$. The time difference between every two observations, i.e., $\Delta_t$, is introduced to a classic RNN model \citep{che2018recurrent}. (2) RNN Decay (RNN-D). An exponential decay process is introduced to capture the dynamical changes of hidden states between observed timestamps in an RNN \citep{mozer2017discrete}. (3) ODE-RNN \citep{rubanova2019latent}. Neural ODE is embedded to learn the dynamical function of hidden states between every two observed timestamps. (4) Neural CDE (NCDE) \citep{kidger2020neural}. Neural CDE creates a continuous data path to control the evolution of the state's ODE flow. (5) ContiFormer (ContiF.) \citep{chen2024contiformer}. ContiFormer generalizes Neural CDE control as a continuous attention mechanism for integrating dynamical information. \textbf{For our proposed NPC framework, we utilize an RNN as the discrete-time model and an ODE-RNN as the continuous-time model}.

\textbf{Implementing details}. For reproducibility, we describe the implementation details in Appendix \ref{sec:implementation}.

\begin{table*}[h]
\caption{Classification test accuracy ($\%$) with mean $\pm$ standard deviation for different baselines.
}
\label{table:class_accu}
\begin{center}
\resizebox{2 \columnwidth}{!}{
\begin{tabular}{l|cccccccccccc}
\toprule
&HAR & Earth   &  ECG&  Car   & WorSyn. & Trace & Plane & Fish & Symbol & SynCon.\\
\midrule
ODE-RNN & $63.1 \pm 0.09$ & $82.3 \pm 0.08$  & $91.1 \pm 0.11$ & $66.7 \pm 0.12$ & $44.5 \pm 0.10$& $97.0 \pm 0.08$& $99.0 \pm 0.09$ & $64.0 \pm 0.12$ & $51.8 \pm 0.13$ & $97.3 \pm 0.11$ \\
RNN-$\Delta_t$ & $60.9 \pm 0.11$ & $81.1 \pm 0.12$ & $90.7 \pm 0.12$ & $46.7 \pm 0.13$ & $43.7 \pm 0.14$ & $69.0 \pm 0.20$ & $83.1 \pm 0.13$ & $65.1 \pm 0.15$ & $85.9 \pm 0.11$ & $96.7 \pm 0.10$ \\
RNN-D  & $55.9 \pm 0.18$ & $82.0 \pm 0.13$  & $58.4 \pm 0.09$ & $21.7 \pm 0.10$  & $46.6 \pm 0.12$ & $98.0 \pm 0.13$ & $77.0\pm 0.14$ & $74.9 \pm 0.09$& $78.4 \pm 0.08$& $94.3 \pm 0.12$\\
NCDE         & $31.8 \pm 0.13$  & $70.5 \pm 0.10$      & $75.7\pm 0.19$ & $25.0 \pm 0.14$ & $24.5 \pm 0.15$ & $59.0 \pm 0.18$& $41.9 \pm 0.11$ & $23.4 \pm 0.09$ & $67.4 \pm 0.13$ & $57.0 \pm 0.12$\\
Contif.  & $58.7 \pm 0.11$     & $82.0 \pm 0.12$   & $\bm{93.9 \pm 0.14}$  & $21.7 \pm 0.21$  & $21.9 \pm 0.11$  & $49.0 \pm 0.17$ &  $96.2 \pm 0.14$  & $12.6 \pm 0.10$ & $85.6 \pm 0.11$  & $89.7 \pm 0.13$ \\
\midrule
NPC & $ \bm{ 70.1 \pm 0.11}$  & $\bm{85.3 \pm 0.09}$ &  $91.1 \pm 0.12$ & $\bm{76.7 \pm 0.10}$ & $\bm{50.6 \pm 0.11}$   & $\bm{99.8 \pm 0.08}$ & $\bm{99.7 \pm 0.11}$ & $\bm{77.7 \pm 0.09}$ & $\bm{86.5 \pm 0.11}$ & $\bm{99.9 \pm 0.07}$ \\
\bottomrule
\end{tabular}
}
\end{center}
\end{table*}

\subsection{Verification of the High Stability in NPC}

For the synthetic dataset in Fig. \ref{fig:syn_train} and \ref{fig:syn_test}, described in Setting, we visualize feature flow $\boldsymbol{h}(t)$ in for ODE-RNN and NPC, shown in Fig. \ref{fig:ht_odernn} and \ref{fig:ht_npc}, respectively. We first analyze the models and observe that, for this binary classification task, ODE-RNN uses a decision boundary of $\boldsymbol{0.17}$, whereas NPC adopts a more standard threshold of $\boldsymbol{0.5}$. Comparing these boundaries with the end-point feature values (i.e., $\boldsymbol{h}(t_{100})$), it's clear that the NPC classifier has a larger classification margin that leads to more robust results. 
 
More specifically, for ODE-RNN's result in Fig. \ref{fig:ht_odernn}, the deviations of test data cause $\boldsymbol{h}(t)$ to change with a high sensitivity. With test deviations, some feature flow's (brown lines) end-point feature value is larger than $0.17$ and wrongly labeled as blue, leading to $87.2\%$ test accuracy. 
For our NPC's result in Fig. \ref{fig:ht_npc}, in the beginning, the blue and the brown features are quickly set apart and converge to equilibrium points, i.e., $1.0$ and $0.0$, respectively. This is because NPC looks ahead with the current dynamics and concludes that features should get stable to equilibrium points to minimize the error. When disturbances appear at timestamp $60$, the impacted features have much smaller deviations compared to those of DOE-RNN, implying that it's harder to leave the equilibrium point. 
Consequently, the brown endpoints are smaller than the boundary, i.e., $0.5$, yielding $100\%$ test accuracy.

\subsection{Stability Guarantees General Classification Improvements on Diversified Domains}

We evaluate the overall performance for time-series classification. $80\%$ of the data is randomly dropped to create irregularly sampled observations. Then, Table \ref{table:class_accu} demonstrates the result. In most cases, our NPC has an accuracy increase of $2\%\sim 15\%$, compared to the state-of-the-art. In particular, when the train/test data have a significant distribution discrepancy, like HAR, CAR, and WorSyn., our methods perform much better ($6\%\sim 15\%$), which indicates the robustness against the data deviations due to the high stability. For the dataset ECG5000, ContiFormer performs slightly better. This could happen because our tested NPC is based on ODE-RNN, which may underperform ContiFormer when the time series is relatively long. However, ContiFormer can be utilized in our NPC framework.

\begin{figure*}[h!]
\centering
\includegraphics[width=2\columnwidth]{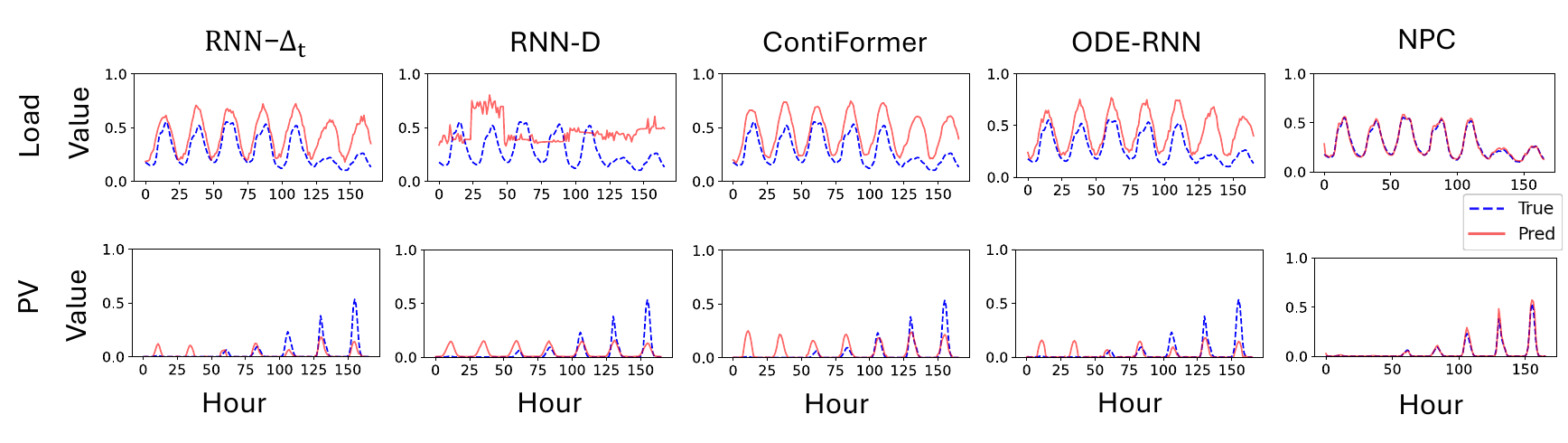}
\centering
\caption{Extrapolation results for $168$h data in a week.}
\label{fig:timeseries_res}
\vspace{-3mm}
\end{figure*}

\subsection{Generalizability Leads to Accurate Interpolation and Extrapolation}
\label{sec:exp_reg}

\begin{figure}[h!]
    \centering
    \subfloat[Interpolation result of ODE-RNN.]{%
        \includegraphics[width = 1\columnwidth, trim=4 0 8 4, clip]{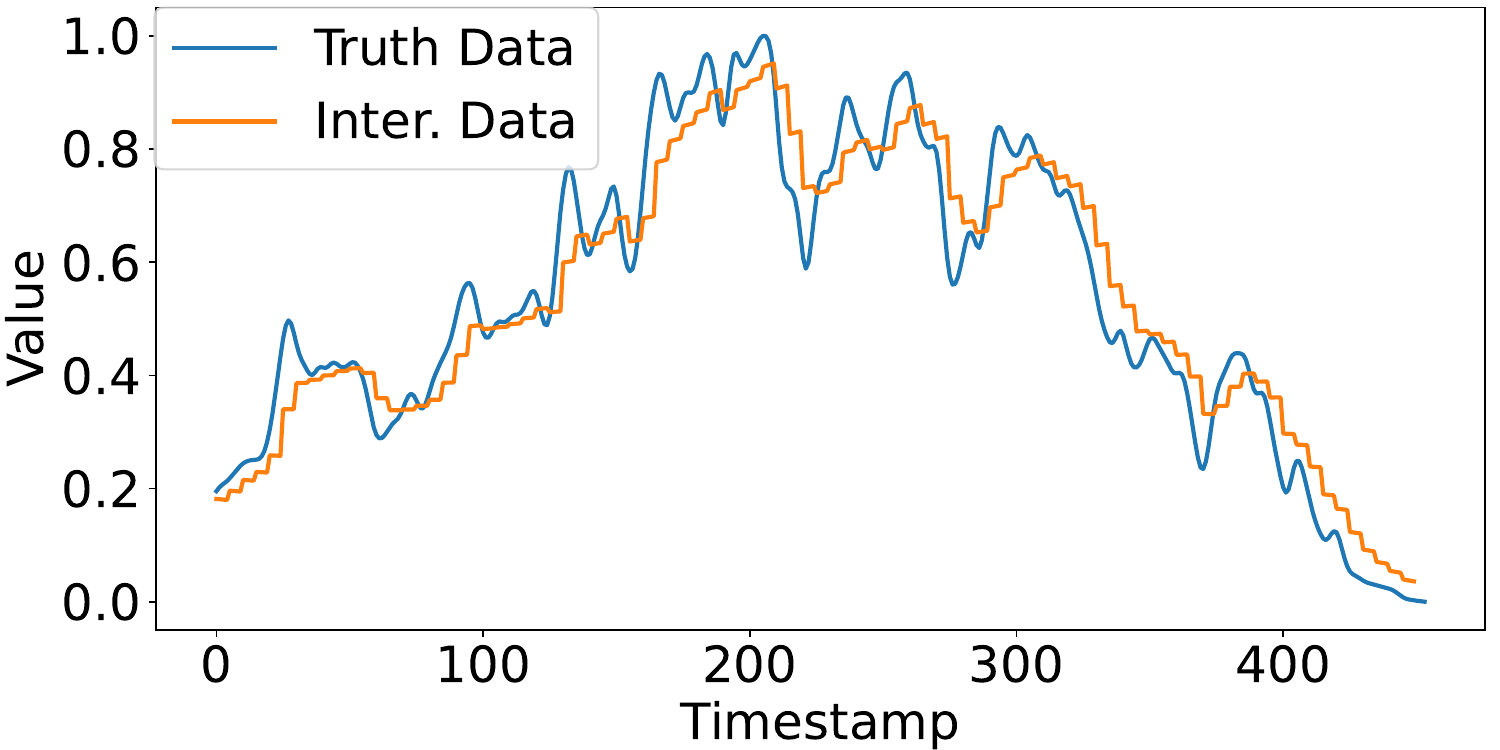}%
        \label{fig:subfig1}%
    }\\
    \subfloat[Interpolation result of NPC.]{%
        \includegraphics[width = 1\columnwidth, trim=4 0 8 4, clip]{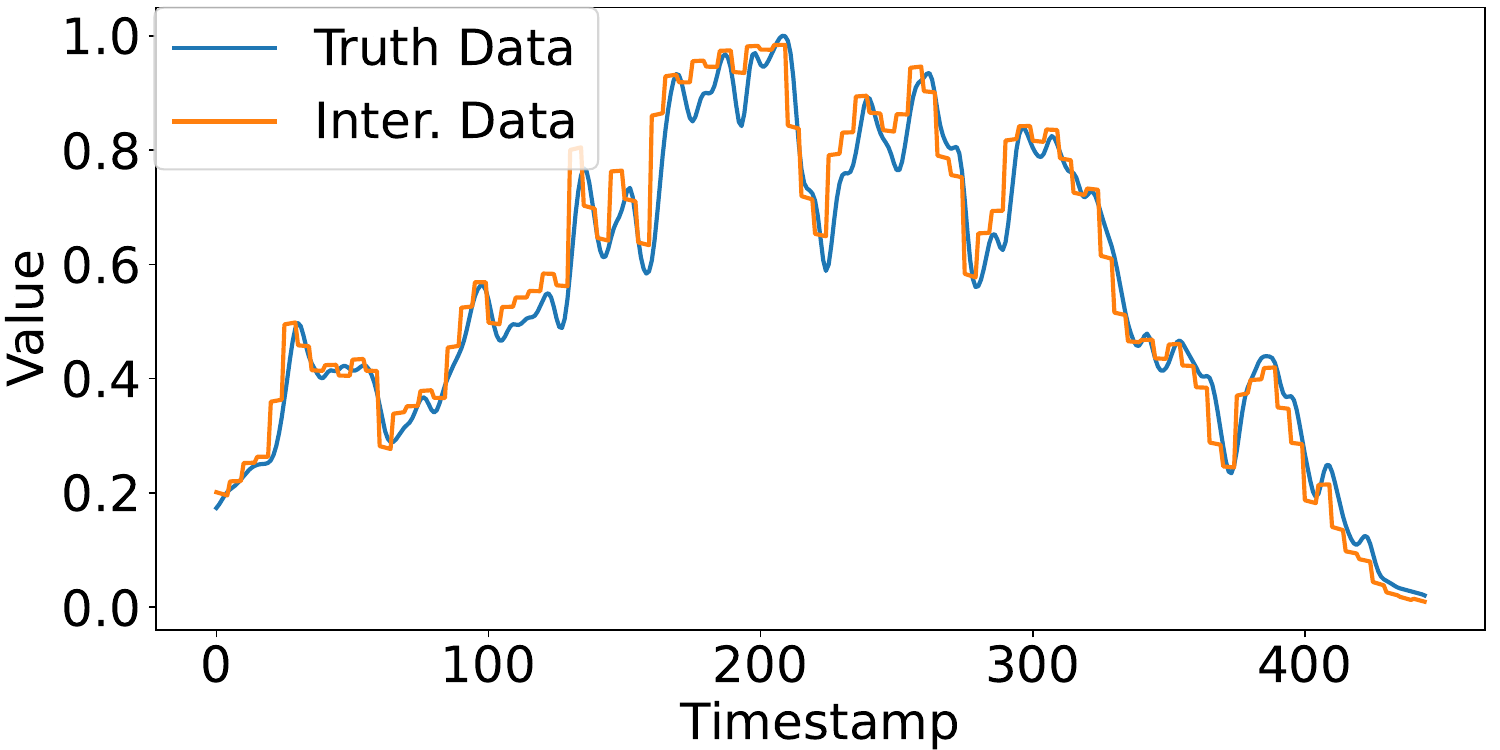}%
        \label{fig:subfig2}%
    }
    \caption{Interpolation results for $420$min data.}
    \label{fig:inter}
\end{figure}

\renewcommand{\arraystretch}{1.05}
\begin{table}[h]
    \centering
    \caption{Interpolation and extrapolation results on PV datasets.}
    \label{tab:inter_extra}
    \resizebox{\columnwidth}{!}{%
    \begin{tabular}{c|c|cccc|c}
        \toprule
        Drop Rate & Metric & ODE-RNN & RNN-$\Delta_t$ & RNN-D & Contif. & NPC \\
        \midrule
        \midrule
        \multicolumn{2}{c|}{\textbf{Interpolation}} & \multicolumn{5}{c}{} \\
        \midrule
        40\% & RMSE       & 0.044 & 0.080 & 0.075 & 0.051 & \textbf{0.037} \\
             & MAPE (\%)  & 2.13  & 3.89  & 3.76  & 2.66  & \textbf{1.79}  \\
        \midrule
        60\% & RMSE       & 0.058 & 0.093 & 0.068 & 0.064 & \textbf{0.041} \\
             & MAPE (\%)  & 2.94  & 4.10  & 3.28  & 3.21  & \textbf{2.21}  \\
        \midrule
        80\% & RMSE       & 0.060 & 0.110 & 0.087 & 0.069 & \textbf{0.051} \\
             & MAPE (\%)  & 3.06  & 4.54  & 4.17  & 3.44  & \textbf{2.37}  \\
        \midrule
        \midrule
        \multicolumn{2}{c|}{\textbf{Extrapolation}} & \multicolumn{5}{c}{} \\
        \midrule
        40\% & RMSE       & 0.071 & 0.088 & 0.092 & 0.033 & \textbf{0.020} \\
             & MAPE (\%)  & 4.17  & 4.86  & 5.21  & 2.25  & \textbf{1.65}  \\
        \midrule
        60\% & RMSE       & 0.075 & 0.094 & 0.098 & 0.048 & \textbf{0.026} \\
             & MAPE (\%)  & 4.46  & 5.22  & 5.51  & 3.04  & \textbf{1.98}  \\
        \midrule
        80\% & RMSE       & 0.089 & 0.120 & 0.100 & 0.055 & \textbf{0.038} \\
             & MAPE (\%)  & 4.86  & 6.53  & 5.73  & 3.19  & \textbf{2.41}  \\
        \bottomrule
    \end{tabular}
    }
\end{table}

Next, we test regression tasks, i.e., interpolation and extrapolation, on PV datasets with a drop rate in $\{40\%, 60\%, 80\%\}$.
When dropping $80\%$ data, Fig. \ref{fig:inter} visualizes the interpolated and the true data using ODE-RNN and our NPC (see Appendix E
for more results) for PV datasets for several hours, where the x-axis unit is minutes. Obviously, the interpolated data from NPC are much closer to the ground truth. This is because, in NPC, the additional discrete RNN adjusts the interpolation of ODE-RNN by minimizing the $M$-horizon predictions. This gives much richer information for current interpolation and approximates the infinite-horizon result. Table \ref{tab:inter_extra} exhibits averaged results for different methods. NPC gains consistent and significant improvements under various drop rates. Moreover, we plot the extrapolation results for PV and load datasets in Fig. \ref{fig:timeseries_res}. They cover a week's data, where the x-axis unit is hours. The results demonstrate that NPC methods can successfully approximate the short-term dynamics and predict the most accurate results. In particular, we note that from $120$h to $168$h (i.e., the weekends), load and PV have gone through a distributional shift. Under this condition, our discrete-time model in NPC can sense the context change and adaptively change the flow evolution in the continuous-time model, thus leading to accurate results. For example, we can observe different dynamics between weekdays and weekends.

\subsection{Sensitivity and Efficiency Analysis}
\label{sec:effi}

We conduct sensitivity analysis with respect to the horizon number $M$. We utilize the interpolation problem as an example and vary $M\in\{2,\cdots, 8\}$. Fig \ref{fig:sensi} illustrates the RMSE with respect to $M$ for the NPC method under different data drop rates. At first glance, the results seemingly violate the theory that the larger $M$ is, the better. In particular, we can observe that the optimal $M$ in Fig \ref{fig:sensi} shifts to the right as the drop rate decreases. We attribute the phenomena to the fact that the denser the data are, the easier it is to learn $f_{\phi}(\cdot)$. Hence, for dense time series, we can assign a larger $M$ to attain convergence in Theorem \ref{theo1} and \ref{theo2}. However, for sparse data, if $M$ is too large, the $M$-horizon minimization and the learning $f_{\phi}(\cdot)$ negatively affect each other at initial iterations due to random parameter initializations. This suggests a novel future direction to improve the training algorithm. 
\begin{figure}[h!]
  \begin{center}
    \includegraphics[width = 1\columnwidth, trim=5 14 6 4, clip]{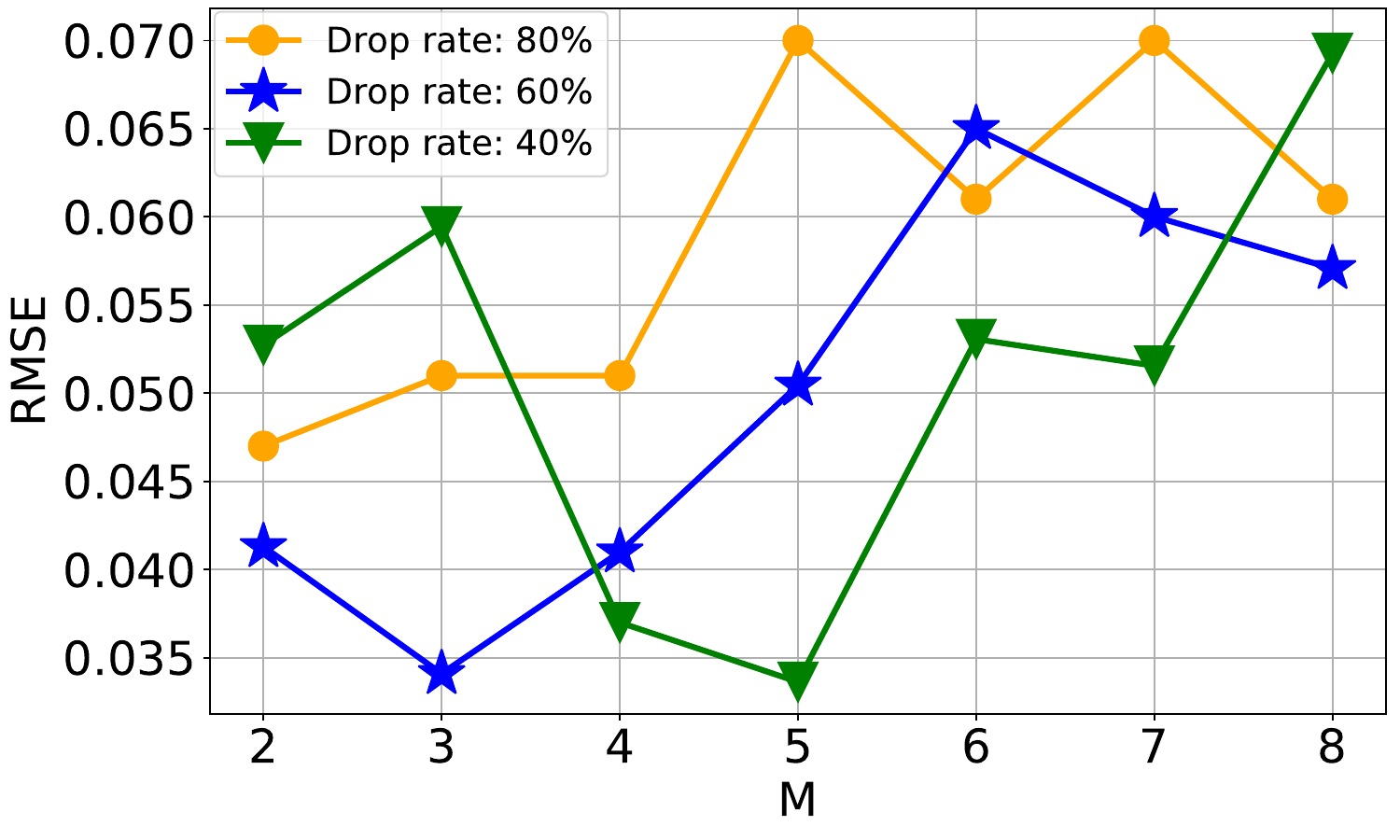}
  \end{center}
  \caption{The sensitivity analysis}
  \label{fig:sensi}
\end{figure}

In the training phase, compared to ODE-RNN, NPC requires around $N(M-1)$ more iterations for solving an ODE. 
Improving the training efficiency is listed as future work in Conclusion. However, in the testing phase, only $1$ out of $M$ action is conducted each time. Hence, NPC has comparable test efficiency. For example, we consider the Car dataset with $(60,73)$ sample numbers and length, respectively. The test time is listed in Table \ref{tab:time_comparison}. NPC is around three times that of ODE-RNN but still affordable for real-time predictions. ContiFormer needs to compute expensive continuous attention and has a much longer test time.



\begin{table}[h]
\caption{Test time (s) for different methods.}
\label{tab:time_comparison}
\begin{center}
\resizebox{\columnwidth}{!}{
    \begin{tabular}{@{}l*{7}{c}@{}}
    \toprule
    Method & RNN-$\Delta_t$ & RNN-D & ODE-RNN & NCDE & ContiF. & NPC  \\
    \midrule
    Time (s) & 0.033 & 0.031 & 0.121 & 0.034 & 3.609 & 0.389  \\
    \bottomrule
\end{tabular}
}
\end{center}
\end{table}

\section{Related Work}
\label{sec:work}

\textbf{Control Theory-based Deep Learning Models}. Several recent studies have shed light on applying insightful control theories to DL. Inspired by the fact that Resnet \citep{he2016deep} and Neural ODE are dynamical functions, their training procedure becomes an optimal control problem where the parameters of neural networks are control variables \citep{rodriguez2022lyanet}. Novel training methods are proposed according to mature control analyses like maximum principle \citep{benning2019deep,li2018maximum,zhang2019you,seidman2020robust}, mean-field theory \citep{liu2019deep,weinan2018mean}, feedback control \citep{chalvidal2020go}, and Lyapunov analysis \citep{rodriguez2022lyanet,kang2021stable}. However, these analyses are hardly applicable to time series. In particular, sequential control with temporal information is needed. To chase optimal sequential control and simultaneously guarantee stability or generalizability, MPC is a natural choice.

\textbf{Deep Learning for Non-linear Control Systems}. There is a different domain that exploits well-trained DL models for non-linear control in physical systems. For instance, the total deviations from a desired trajectory should be minimized in a vehicle trajectory tracking problem \citep{zhang2021trajectory}. In these systems, DL models are utilized to approximate unknown system dynamics for MPC formulations \citep{zhang2021trajectory,luo2023model,pan2011model,chee2022knode}. This suggests the high potential of combining MPC and DL.

\section{Conclusion, Limitation, and Future Work}
\label{sec:con}

We propose NPC, a novel method to coordinate arbitrary discrete- and continuous-time DL models to efficiently utilize short-term dynamics and long-term patterns. Moreover, our model can provably achieve high stability and generalizability. The target is an infinite-horizon cost minimization. Grounded in the optimal control theory, we apply the finite-horizon relaxation with an exponential convergence. Furthermore, a feedback mechanism is added as additional information to improve the convergence. The framework is simple yet highly effective, gaining the best performance on various time-series tasks. The limitation of NPC is the high training time due to the additional $\mathcal{O}(N\cdot M)$ ODE computations and minimization. In the future, we could address the issue by $(1)$ selectively conducting the $M$-horizon optimization and $(2)$ employing SSM instead of Neural ODE to capture the dynamics in the continuous-time model. Then, fast ODE inferences could be achieved based on Kalman filter \citep{schirmer2022modeling,ansari2023neural} or Fast Fourier Transform (FFT) \citep{fu2022hungry}.

\bibliographystyle{ACM-Reference-Format}
\bibliography{acmart.bib}

\appendix

\section{Assumption}
\label{appendix:assum}
\begin{assumption} 
\label{assum1}
Let $\frac{d\boldsymbol{h}(t)}{dt} = f(\boldsymbol{h}(t),\boldsymbol{u}(t))$ denote the true dynamics of $\boldsymbol{h}(t)$ in our Coordinated model. Consider linearization $f(\boldsymbol{h}(t),\boldsymbol{u}(t))\approx A\boldsymbol{h}(t)+B\boldsymbol{u}(t)$ and $\ell_{\phi}^2(\boldsymbol{h}(t)) \approx C\boldsymbol{h}(t)$, where $A$, $B$, and $C$ are the state, input, and output matrices of the state-space model (SSM), respectively. Assume the following conditions hold:
\begin{itemize}
\item Our approximation for the true dynamics satisfies
$f_{\phi^*}(\boldsymbol{h}(t),\allowbreak\boldsymbol{u}(t))
 = f(\boldsymbol{h}(t),\allowbreak\boldsymbol{u}(t))
 + \boldsymbol{w}(t)$, 
where $\boldsymbol{w}(t)$ is an approximation error.
\item The optimal control in NPC training is sufficiently close to the $M$-horizon optimal control results in the SSM model.
\item $(A,B)$ is controllable and $(A,C)$ is observable.
\item There exists a Lipschitz constant $L$ such that $\forall i,j>0$,
\begin{align*}
& ||f(\boldsymbol{h}(t_i),\boldsymbol{u}(t_j)) - A\boldsymbol{h}(t_i)-B\boldsymbol{u}(t_j) -f(\boldsymbol{h}(t_j),\boldsymbol{u}(t_j)) +  A\boldsymbol{h}(t_j) \\ 
&  +B\boldsymbol{u}(t_j)|| \leq L(||\boldsymbol{h}(t_i)-\boldsymbol{h}(t_j)||+||\boldsymbol{u}(t_i)-\boldsymbol{u}(t_j)||).
\end{align*}
\end{itemize}
\end{assumption}

The first condition depends on the approximation power of the continuous-time model can improve the approximation to make $\boldsymbol{w}(t)$ sufficiently small. The second to the fourth conditions require a small linearization error from $f(\cdot)$ to an SSM. While it's hard to give direct proofs, a large amount of work \citep{gu2021efficiently,smith2022simplified,schirmer2022modeling,ansari2023neural,gu2022parameterization} reveals that SSM is competent for time-series modeling. 

\section{Theorem 1 and Proof} \label{appendix:proof1}

\begin{theorem*}[Stability]
\label{theo1_appen}
Let $T=t_{i+M}-t_i$, $\tau=t_{i+1}-t_i$ and $\boldsymbol{h}_{(\psi^*,\phi^*)}^{*}(t)$ denote the trajectory after NPC training and control. There exists constants $K$, $K_1$, $K_2$, $\mu_{\infty}$, and $M_{\infty}\geq 1$ such that 
\begin{align}
\label{eqn:stability_appen}
& ||\boldsymbol{h}_{(\psi^*,\phi^*)}(t)|| \leq  M_{\infty}e^{-\mu t}||\boldsymbol{h}_{(\psi^*,\phi^*)}^*(t_1)|| 
\notag
\\ & + \frac{1-e^{-\mu t}}{\mu}K(1+(L+1)\tau e^{K(L+1)\tau})||\boldsymbol{w}||_{l^{\infty}(0,t)},
\end{align}
where $\mu=\mu_{\infty}-K_1e^{-2\mu_{\infty}(T-\tau)}-K_2L-KL(L+1)\tau e^{K(L+1)\tau}$ and $||\cdot||_{l^{\infty}(0,t)}$ is the infinity norm on the function space over $(0,t)$.
\end{theorem*}
\begin{proof}
To begin with, we introduce preliminary results about the infinite-horizon and $M$-horizon MPC for SSM. Then, we conduct the stability and convergence analysis for nonlinear dynamical equations. First, by the SSM formula in Assumption 1,
we consider a continuous extension of the $M$-horizon minimization in Equation 6 
in $[t_i,t_i+T]$:
\begin{equation}
\label{eqn:J_T_Mhorizon}
J_T(\boldsymbol{u};t_i) =  \frac{1}{2}\int_{t_i}^{t_i+T}||C\boldsymbol{h}(t)||^2 + (\boldsymbol{u}(t))^{\top} R \boldsymbol{u}(t)dt,
\end{equation}
where $||C\boldsymbol{h}(t)||^2$ is the continuous version of $J(\boldsymbol{h}(t_i),\cdots,\boldsymbol{h}(t_{i+M}))$ to force $C\boldsymbol{h}(t)\rightarrow \boldsymbol{0}$ so that the feature state is classified to have label $y=0$. $(\boldsymbol{u}(t))^{\top} R \boldsymbol{u}(t)$ is a continuous version of $\lambda \hat{J}(\boldsymbol{u}_i,\cdots,\boldsymbol{u}_{i+M})$ and $R$ is a symmetric positive
definite matrix. Similarly, for the infinite-horizon problem \citep{veldman2022local}, we have:
\begin{equation}
\label{eqn:J_infty}
J_{\infty}(\boldsymbol{u}) =  \frac{1}{2}\int_{t_1}^{\infty}||C\boldsymbol{h}(t)||^2 + (\boldsymbol{u}(t))^{\top} R \boldsymbol{u}(t)dt,
\end{equation}
where $t_1$ is the start time in our time-series observations and we can set $t_1=0$ without loss of generality. By Assumption 1 in the main paper, in the SSM, the above optimizations have the following constraints:
\begin{equation}
\begin{aligned}
\label{eqn:SSM_constraint}
\frac{d\boldsymbol{h}(t)}{dt} &=  A\boldsymbol{h}(t)+B\boldsymbol{u}(t), \\ \ell_{\phi}^2(\boldsymbol{h}(t)) &= C\boldsymbol{h}(t),
\end{aligned}
\end{equation}
where $(A,B)$ is controllable and $(A,C)$ is observable. Hence, for the objective in Equation \eqref{eqn:J_infty} and the constraint \eqref{eqn:SSM_constraint}, the optimal trajectory of \citep{sage1968optimum} is given by:
\begin{equation}
\begin{aligned}
\label{eqn:optimal_traj_hat_h_infty}
\frac{d\tilde{\boldsymbol{h}}^*_{\infty}(t)}{dt} &= A_{\infty}\tilde{\boldsymbol{h}}^*_{\infty}(t), \tilde{\boldsymbol{h}}^*_{\infty}(t_1)=\boldsymbol{h}(t_1),\\
A_{\infty} &= A - BR^{-1}B^{\top}P_{\infty},\\
\tilde{\boldsymbol{u}}^*_{\infty}(t) &= -R^{-1}B^{\top} P_{\infty}\tilde{\boldsymbol{h}}^*_{\infty}(t),
\end{aligned}
\end{equation}
where $P_{\infty}$ is the unique symmetric positive-definite solution of the following Algebraic Riccati Equation (ARE) \citep{lancaster1995algebraic}:
\begin{equation}
A^{\top} P_{\infty} + P_{\infty} A - P_{\infty} B R^{-1}B^{\top}P_{\infty} + C^{\top} C = \boldsymbol{0}.
\end{equation}

According to \citep{porretta2013long}, the controllability of $(A,B)$ leads to the fact that there are constants $\mu_{\infty}>0$ and $M_{\infty}\geq 1$ such that
\begin{equation}
\label{eqn:upper_operator}
\forall t\geq t_1, ||e^{A_{\infty}t}||\leq M_{\infty}e^{-\mu_{\infty}t},
\end{equation}
where $||A||$ for an operator $A$ is the operator norm. The exponentially decreasing upper bound in Equation \eqref{eqn:upper_operator} implies that if $t\rightarrow \infty$, the solution in Equation \eqref{eqn:optimal_traj_hat_h_infty}, i.e., $\tilde{\boldsymbol{h}}^*_{\infty}(t)\rightarrow \boldsymbol{0}$. This implies that the infinite-horizon problem has exponential convergence to the origin. 

Then, for the $M$-horizon problem in Equation \eqref{eqn:J_T_Mhorizon}, by \citep{sage1968optimum}, the optimal trajectory is: 
\begin{equation}
\label{eqn:optimal_traj_hat_h}
\begin{aligned}
\frac{d\tilde{\boldsymbol{h}}^*_{T}(t)}{dt} &= A_{T,\tau}(t)\tilde{\boldsymbol{h}}^*_{T}(t), \tilde{\boldsymbol{h}}^*_{T}(t_1)=\boldsymbol{h}(t_1),\\
A_{T,\tau} &= A - BR^{-1}B^{\top}P(T-(t \ \text{mod}\ \tau)),\\
\tilde{\boldsymbol{u}}^*_{T}(t) &= -R^{-1}B^{\top} P(T-(t \ \text{mod}\ \tau))\tilde{\boldsymbol{h}}^*_{T}(t),
\end{aligned}
\end{equation}

$A_{T,\tau}$ is a $\tau$-periodic matrix since the MPC only updates the optimal action from Equation \eqref{eqn:J_T_Mhorizon} and evolve the optimal state for an interval of $\tau$ (i.e., conduct one-horizon action in our Algorithm 1 in the main paper). Subsequently, Equation \eqref{eqn:J_T_Mhorizon} needs to be resolved. Essentially, $P(t)$ follows the so-called Ricatti Differential Equation (RDE) \citep{benner2018numerical} in $[0,T]$: 
\begin{align}
\frac{dP(t)}{dt} & = A^{\top} P(t) + P(t)A - P(t)BR^{-1}B^{\top}P(t) + C^{\top}C, \notag \\ 
& P(0)=C.
\end{align}

To analyze the convergence of $\tilde{\boldsymbol{h}}^*_{T}(t)$, it follows that we need to understand the relations between the RDE solution $P(t)$ and the ARE solution $P_{\infty}$. By \citep{callier1994convergence,porretta2013long,veldman2022local}, there exists a constant $K_0$ such that 
\begin{equation}
\label{eqn:pt_pinfty_conver}
||P(t) - P_{\infty}||\leq K_0e^{-2\mu_{\infty}t}.
\end{equation}
Therefore, as $t\rightarrow \infty$, $P(t)\rightarrow P_{\infty}$ and $A_{T,\tau}(t)\rightarrow A_{\infty}$ when $T-\tau \rightarrow \infty$. This connects the convergence analysis between infinite-horizon to $M$-horizon results with SSM as the ODE model. 

Instead of the SSM in Equation \eqref{eqn:SSM_constraint}, in our NPC framework, we note that the constraint should satisfy:
\begin{equation}
\label{eqn:nonlinear_constraint}
\frac{d\boldsymbol{h}(t)}{dt}=f_{\phi^*}(\boldsymbol{h}(t),\boldsymbol{u}(t))=f(\boldsymbol{h}(t),\boldsymbol{u}(t))+\boldsymbol{w}(t),
\end{equation}
where $\boldsymbol{w}(t)$ is a sufficiently small approximation error by the first condition in Assumption 1. 
For our NPC training and by the second condition in Assumption 
1, we have:
\begin{equation}
\begin{aligned}
\label{eqn:h_optimal_npc}
\frac{d\boldsymbol{h}_{(\psi^*,\phi^*)}^*(t)}{dt}&=f(\boldsymbol{h}_{(\psi^*,\phi^*)}^*(t),\tilde{\boldsymbol{u}}_T^*(t)) + \boldsymbol{w}(t)\\
&= f(\boldsymbol{h}_{(\psi^*,\phi^*)}^*(t),\tilde{\boldsymbol{u}}_T^*(t)) \\ & +  \boldsymbol{w}(t) - A\boldsymbol{h}_{(\psi^*,\phi^*)}^*(t)- B\tilde{\boldsymbol{u}}_T^*(t)\\
&-BR^{-1}B^{\top}P(T-(t \ \text{mod}\ \tau))\tilde{\boldsymbol{h}}^*_{T}(t) + \\ &BR^{-1}B^{\top}P(T-(t \ \text{mod}\ \tau))\boldsymbol{h}_{(\psi^*,\phi^*)}^*(t)\\
&+A_{T,\tau}(t)\boldsymbol{h}_{(\psi^*,\phi^*)}^*(t) - A_{\infty}\boldsymbol{h}_{(\psi^*,\phi^*)}^*(t) \\ & + A_{\infty}\boldsymbol{h}_{(\psi^*,\phi^*)}^*(t)\\
&=f(\boldsymbol{h}_{(\psi^*,\phi^*)}^*(t),\tilde{\boldsymbol{u}}_T^*(t))  - A\boldsymbol{h}_{(\psi^*,\phi^*)}^*(t) \\ 
& - B\tilde{\boldsymbol{u}}_T^*(t) + \boldsymbol{w}(t)
+A_{\infty}\boldsymbol{h}_{(\psi^*,\phi^*)}^*(t) \\ 
& + (A_{T,\tau}(t)-A_{\infty})\boldsymbol{h}_{(\psi^*,\phi^*)}^*(t) - \\ &BR^{-1}B^{\top}P(T-(t \ \text{mod}\ \tau))\boldsymbol{\epsilon}(t),
\end{aligned}
\end{equation}
where $\boldsymbol{\epsilon}(t)=\tilde{\boldsymbol{h}}^*_{T}(t)-\boldsymbol{h}_{(\psi^*,\phi^*)}^*(t)$. Equation \eqref{eqn:h_optimal_npc} links the $\boldsymbol{h}_{(\psi^*,\phi^*)}^*(t)$ and $\tilde{\boldsymbol{h}}^*_{T}(t)$ so that the convergence of $\boldsymbol{h}_{(\psi^*,\phi^*)}^*(t)$ can be analyzed. Specifically, by the forth condition of Assumption 1 
, we have:
\begin{equation}
\begin{aligned}
\label{eqn:norm_L_K2}
&||f(\boldsymbol{h}_{(\psi^*,\phi^*)}^*(t),\tilde{\boldsymbol{u}}_T^*(t))  - A\boldsymbol{h}_{(\psi^*,\phi^*)}^*(t)- B\tilde{\boldsymbol{u}}_T^*(t)|| \\ 
& \leq L(||\boldsymbol{h}_{(\psi^*,\phi^*)}^*(t)||+||\tilde{\boldsymbol{u}}_T^*(t)||)\\
&=L(||\boldsymbol{h}_{(\psi^*,\phi^*)}^*(t)||+ \\ & ||-R^{-1}B^{\top}P(T-(t \ \text{mod}\ \tau))(\boldsymbol{\epsilon}(t)+\boldsymbol{h}_{(\psi^*,\phi^*)}^*(t))||)\\
&\leq L((1+K_2')||\boldsymbol{h}_{(\psi^*,\phi^*)}^*(t)||+K_2'||\boldsymbol{\epsilon}(t)||),
\end{aligned}
\end{equation}
where $K_2'=||R^{-1}B^{\top}||(K_0+||P_{\infty}||)$ and the last inequality holds by the fact that $||P(T-(t \ \text{mod}\ \tau))||\leq K_0 + ||P_{\infty}||$ from Equation \eqref{eqn:pt_pinfty_conver}. Applying the variation of constants formula \citep{ball1977strongly} to Equation \eqref{eqn:h_optimal_npc} and the norm operations, by Equations \eqref{eqn:upper_operator} and \eqref{eqn:norm_L_K2}, we have:
\begin{equation}
\begin{aligned}
\label{eqn:norm_h_star}
||\boldsymbol{h}_{(\psi^*,\phi^*)}^*(t)||& \leq M_{\infty}e^{-\mu_{\infty}}t||\boldsymbol{h}_{(\psi^*,\phi^*)}^*(t_1)||+ \\& (K_1e^{-2\mu_{\infty}(T-\tau)} + K_2L)\int_{0}^te^{-\mu_{\infty}(t-s)} \\& ||\boldsymbol{h}_{(\psi^*,\phi^*)}^*(s)||ds\\
&+K(L+1)\int_0^te^{-\mu_{\infty}(t-s)}||\boldsymbol{\epsilon}(s)||ds \\& + M_{\infty}\int_0^te^{-\mu_{\infty}(t-s)}||\boldsymbol{w}(s)||ds,
\end{aligned}
\end{equation}
where $K_2=M_{\infty}(1+K_2')$ and $K_1=M_{\infty}||BR^{-1}B^{\top}||K_0$. Note that the above inequality also uses the inequality: $\max_t||A_{T,\tau}(t)-A_{\infty}||\leq ||BR^{-1}B^{\top}||K_0e^{-2\mu_{\infty}(T-\tau)}$ by Equation \eqref{eqn:pt_pinfty_conver}. To investigate the impact of $||\boldsymbol{\epsilon}(t)||$ in Equation \eqref{eqn:norm_h_star}, the definition of $\boldsymbol{\epsilon}(t)$ and Equation \eqref{eqn:h_optimal_npc} imply that:
\begin{equation}
\begin{aligned}
\frac{d\boldsymbol{\epsilon}(t)}{dt}&= A_{\infty}\boldsymbol{\epsilon}(t)+(A_{T,\tau}(t)-A_{\infty})\boldsymbol{\epsilon}(t) \\
& -BR^{-1}B^{\top}P(T-(t \ \text{mod}\ \tau))\boldsymbol{\epsilon}(t)\\
&-f(\boldsymbol{h}_{(\psi^*,\phi^*)}^*(t),\tilde{\boldsymbol{u}}_T^*(t)) + A\boldsymbol{h}_{(\psi^*,\phi^*)}^*(t) \\
& + B \tilde{\boldsymbol{u}}_T^*(t) - \boldsymbol{w}(t).
\end{aligned}
\end{equation}

Using the variation of constants formula again and Gronwall lemma \citep{scheutzow2013stochastic} gives:
\begin{equation}
\begin{aligned}
\label{eqn:norm_eplison}
& ||\boldsymbol{\epsilon}(t)||\leq e^{K(L+1)\tau}(KL||\boldsymbol{h}_{(\psi^*,\phi^*)}^*(t)||_{l^1(0,t)} \\
& +\tau M_{\infty}||\boldsymbol{w}||_{L^{\infty}(0,t)}).
\end{aligned}
\end{equation}

Combing Equation \eqref{eqn:norm_h_star} and \eqref{eqn:norm_eplison} finally yields
\begin{equation}
\begin{aligned}
& ||\boldsymbol{h}_{(\psi^*,\phi^*)}(t)||\leq  M_{\infty}e^{-\mu t}||\boldsymbol{h}_{(\psi^*,\phi^*)}^*(t_1)|| \\
& + \frac{1-e^{-\mu t}}{\mu}K(1+(L+1)\tau e^{K(L+1)\tau})||\boldsymbol{w}||_{l^{\infty}(0,t)},
\end{aligned}
\end{equation}
where $\mu=\mu_{\infty}-K_1e^{-2\mu_{\infty}(T-\tau)}-K_2L-KL(L+1)\tau e^{K(L+1)\tau}$. 
\end{proof}

\section{Theorem 2 and Proof}\label{appendix:proof2}
\begin{theorem*}[Generalizability]
    \label{theo2_appen} Consider $T$, $\tau$, $K_2$, $\mu_{\infty}$, and $\boldsymbol{h}_{(\psi^*,\phi^*)}(t)$ defined in Theorem 1 
    and let $\boldsymbol{u}_{(\psi^*,\phi^*)}(t)$ denote the optimal control action after NPC training in Algorithm 1
    . Let $\boldsymbol{u}^*_{\infty}(t)$ denote the optimal solution of applying the linear model in Assumption 1 
    to the infinite-horizon minimization problem, defined in Equation \eqref{eqn:J_infty} in Appendix \ref{appendix:proof1}. Let $\boldsymbol{h}^*_{\infty}(t)$ denote the state controlled by $\boldsymbol{u}^*_{\infty}(t)$ using the nonlinear model $f_{\phi^*}(\cdot)$. There exists a constant $K_3$ such that:
    \begin{equation}
    \begin{aligned}
    \label{eqn:generalizability_appen}
& ||\boldsymbol{h}_{(\psi^*,\phi^*)}(t) - \boldsymbol{h}^*_{\infty}(t)|| + ||\boldsymbol{u}_{(\psi^*,\phi^*)}(t) - \boldsymbol{u}^*_{\infty}(t)|| \\
& \leq K_3e^{-2\mu_{\infty}(T-\tau)}\Big(\frac{L+1}{\mu_{\infty}-K_2L}||\boldsymbol{h}||_{l^1(0,t)} +||\boldsymbol{h}(t)||\Big)  \\
& + K_3\tau e^{K_3(L+1)\tau}\frac{L+1}{\mu_{\infty}-K_2L}(L||\boldsymbol{h}||_{l^1(0,t)}+||\boldsymbol{w}||_{l^{\infty}(0,t)}),
\end{aligned}
    \end{equation}
where $||\cdot||_{l^{1}(0,t)}$ is the $l$-$1$ norm on the function space over $(0,t)$.
\end{theorem*}

\begin{proof}
Recall that the infinite-horizon cost minimization is in Equation \eqref{eqn:J_infty} with the SSM model gives optimal state and control action trajectories in Equation \eqref{eqn:optimal_traj_hat_h_infty}. By definition, $\boldsymbol{u}_{\infty}^*(t)=\tilde{\boldsymbol{u}}^*_{\infty}(t)=R^{-1}B^{\top} P(T-(t \ \text{mod}\ \tau))\tilde{\boldsymbol{h}}^*_{T}(t)$ in Equation \eqref{eqn:optimal_traj_hat_h_infty}. Thus, we have:
\begin{equation}
\begin{aligned}
\label{eqn:h_infty_star}
\frac{d\boldsymbol{h}_{\infty}^*(t)}{dt} & =f(\boldsymbol{h}^*_{\infty}(t),\boldsymbol{u}_{\infty}^*(t)) + \boldsymbol{w}(t)+A_{\infty}\boldsymbol{h}^*_{\infty}(t) \\
& - A\boldsymbol{h}^*_{\infty}(t) - B\boldsymbol{u}_{\infty}^*(t),
\end{aligned}
\end{equation}
where $A_{\infty} = A - BR^{-1}B^{\top}P_{\infty}$ is defined in Equation \eqref{eqn:optimal_traj_hat_h_infty}. Then, we evaluate the difference $\boldsymbol{e}(t)=\boldsymbol{h}_{(\psi^*,\phi^*)}(t) - \boldsymbol{h}^*_{\infty}(t)$ by Equations \eqref{eqn:h_infty_star} and \eqref{eqn:h_optimal_npc}. 

\begin{equation}
\begin{aligned}
\label{eqn:e_state}
\frac{d\boldsymbol{e}(t)}{dt}&=A_{\infty}\boldsymbol{e}(t)+(A_{T,\tau}(t)-A_{\infty})\boldsymbol{h}_{(\psi^*,\phi^*)}(t) \\
&-BR^{-1}B^{\top}P(T-(t \ \text{mod}\ \tau))\boldsymbol{\epsilon}(t)\\
&+f(\boldsymbol{h}_{(\psi^*,\phi^*)}(t),\tilde{\boldsymbol{u}}_{T}^*(t)) - f(\boldsymbol{h}^*_{\infty}(t),\boldsymbol{u}_{\infty}^*(t)) \\
& - A\boldsymbol{e}(t) - B(\tilde{\boldsymbol{u}}_{T}^*(t)-\boldsymbol{u}_{\infty}^*(t)).
\end{aligned}
\end{equation}

By the last condition in Assumption 1,
\begin{equation}
\begin{aligned}
\label{eqn:norm_f_infty}
&||f(\boldsymbol{h}_{(\psi^*,\phi^*)}(t),\tilde{\boldsymbol{u}}_{T}^*(t)) - f(\boldsymbol{h}^*_{\infty}(t),\boldsymbol{u}_{\infty}^*(t)) - A\boldsymbol{e}(t) \\
&- B(\tilde{\boldsymbol{u}}_{T}^*(t)-\boldsymbol{u}_{\infty}^*(t))|| \\
&\leq L(||\boldsymbol{e}(t)||+||\tilde{\boldsymbol{u}}_{T}^*(t)-\boldsymbol{u}_{\infty}^*(t)||)\\
&\leq L((1+K_2')||\boldsymbol{e}(t)||+K_2'||\boldsymbol{\epsilon}(t)|| \\
&+Ke^{-2\mu_{\infty}(T-\tau)}||\boldsymbol{h}_{(\psi^*,\phi^*)}(t)||),
\end{aligned}
\end{equation}
where $K_2'=||R^{-1}B^{\top}||(K_0+||P_{\infty}||)$ and $K_0$ and $P_{\infty}$ are defined in Equation \eqref{eqn:pt_pinfty_conver}. The last inequality holds because we have:
\begin{equation}
\begin{aligned}
\label{eqn:uhatstar_uinfty}
&\tilde{\boldsymbol{u}}_{T}^*(t)-\boldsymbol{u}_{\infty}^*(t) \\
& =-R^{-1}B^{\top}(P(T-(t \ \text{mod}\ \tau))\tilde{\boldsymbol{h}}^*_{T}(t)-P_{\infty}\boldsymbol{h}_{\infty}^*(t))\\
&=-R^{-1}B^{\top}(P(T-(t \ \text{mod}\ \tau))\boldsymbol{\epsilon}(t)+ \\
& \quad (P(T-(t \ \text{mod}\ \tau))-P_{\infty})\boldsymbol{h}_{(\psi^*,\phi^*)}(t)+P_{\infty}\boldsymbol{e}(t)).
\end{aligned}
\end{equation}

Then, we apply the variation of constant formula to Equation \eqref{eqn:e_state} and take norms. By Equation \eqref{eqn:norm_f_infty}, we have:
\begin{equation}
\begin{aligned}
||\boldsymbol{e}(t)|| & \leq K_2L\int_0^t e^{-\mu_{\infty}(t-s)}||\boldsymbol{e}(s)||ds + \\
& \quad K(L+1)e^{-2\mu_{\infty}(T-\tau)}||\boldsymbol{h}_{(\psi^*,\phi^*)}(t)||_{l^1(0,t)}\\
&+K(L+1)\int_0^te^{-\mu_{\infty}(t-s)}||\boldsymbol{\epsilon}(s)||ds \\
&\leq K_2L\int_0^t e^{-\mu_{\infty}(t-s)}||\boldsymbol{e}(s)||ds \\
&+ K(L+1)((e^{-2\mu_{\infty}(T-\tau)}+L\tau e^{K(L+1)\tau}) \\
& \quad ||\boldsymbol{h}_{(\psi^*,\phi^*)}(t)||_{l^1(0,t)}+\tau e^{K(L+1)\tau}||\boldsymbol{w}||_{l^{\infty}(0,t)}) \\
&\leq  \frac{1}{\mu_{\infty}-K_2L}K(L+1) ((e^{-2\mu_{\infty}(T-\tau)}+L\tau e^{K(L+1)\tau}) \\
& \quad ||\boldsymbol{h}_{(\psi^*,\phi^*)}(t)||_{l^1(0,t)}+\tau e^{K(L+1)\tau}||\boldsymbol{w}||_{l^{\infty}(0,t)})
\end{aligned}
\end{equation}
where $K_2=M_{\infty}(1+K_2')$ and the last second and the last inequalities is derived by Gronwall’s lemma. Similarly, we can take norm of Equation \eqref{eqn:uhatstar_uinfty} and finally obtain:
    \begin{equation}
    \begin{aligned}
    \label{eqn:generalizability_appen2}
& ||\boldsymbol{h}_{(\psi^*,\phi^*)}(t) - \boldsymbol{h}^*_{\infty}(t)|| + ||\boldsymbol{u}_{(\psi^*,\phi^*)}(t) - \boldsymbol{u}^*_{\infty}(t)|| \\
& \leq K_3e^{-2\mu_{\infty}(T-\tau)}\Big(\frac{L+1}{\mu_{\infty}-K_2L}||\boldsymbol{h}||_{l^1(0,t)} +||\boldsymbol{h}(t)||\Big) \\
&+ K_3\tau e^{K_3(L+1)\tau}\frac{L+1}{\mu_{\infty}-K_2L}(L||\boldsymbol{h}||_{l^1(0,t)}+||\boldsymbol{w}||_{l^{\infty}(0,t)}),
\end{aligned}
    \end{equation}
where $\boldsymbol{u}_{(\psi^*,\phi^*)}(t) \approx \tilde{\boldsymbol{u}}_{T}^*(t)$ by the second condition in Assumption 1. 
\end{proof}




\section{Appendix C: Synthetic Data Generation}
\label{sec:synthe_data}
We generate a binary time series classification synthetic training data and test data, separately.
For training dataset, it has 50 samples per class, each with 100 time steps. 
The data is created by applying sine and cosine functions to a sequence of 100 evenly spaced time steps between 0 and 6 and adding random noise. Mathematically,
\begin{equation}
\begin{cases}
    y^{class_0} = 7 + \sin(t) + \cos(t) + 0.2 \cdot N, \\
    y^{class_1} = 2 \sin(t) + 2 \cos(t) + 0.2 \cdot N
\end{cases}
\label{toy_functions}
\end{equation}
where $N$ is the random noise sampled from a normal distribution $\mathcal{N}(0,1)$. The generated data for both classes is combined, reshaped to include a singleton dimension, and labeled, creating an array of shapes of $(100, 100, 1)$ for the time series data and a corresponding label array. The data is then shuffled, normalized, and converted to $float 32$ type, while labels are converted to $int 64$ type. 

For the test dataset, we consider a more complex scenario.
Class 0 has 20 different types of time series patterns, each repeated 50 times, resulting in a total of 1000 samples. The time steps are generated as 100 evenly spaced values between 0 and 6.
To introduce variability, for each of the 20 Class 0 types, a parabolic abnormal noise effect is calculated and added to the base sine and cosine waveform. The specific process of adding the parabolic abnormal noise to class 0 is shown in Algorithm \ref{alg:toy_test_dataset}.
\begin{algorithm}
\caption{Data Generation with Parabolic Noise}
\begin{algorithmic}[1]
\State $n \gets 50$  \Comment{Number of samples per type}
\State $kinds \gets 20$ \Comment{Number of different Class 0 types}
\State $idx_1, idx_2 \gets 60, 100$ \Comment{Indices for noise time range}
\State $time\_steps \gets \text{linspace}(0, 6, 100)$
\State $y^{class_0} \gets \sin(time\_steps) + \cos(time\_steps) + \mathcal{N}(0, 1) + 7$
\State $y^{class_1} \gets 2 \cdot \sin(time\_steps) + 2 \cdot \cos(time\_steps) + \mathcal{N}(0,1)$

\State $x_{pass} \gets \frac{idx_1}{n\_steps} \cdot 6$
\State $y_{pass} \gets class\_0[0, idx_1]$
\State $h \gets \frac{idx_1 + idx_2}{2} \cdot \frac{6}{n\_steps}$  \Comment{Axis of symmetry}
\State $t \gets time\_steps[idx_1:idx_2]$  \Comment{Noise time range}
\State $class\_noise \gets \text{zeros}(kinds, n\_steps)$

\For{$i \gets 0$ to $kinds-1$}
    \State $k \gets y_{pass} - 0.3 \cdot i \cdot (x_{pass} - h)^2$  \Comment{Ensure passing through $(x_{pass}, y_{pass})$}
    \State $y \gets 0.3 \cdot i \cdot (t - h)^2 + k$
    \State $noise \gets \text{zeros}(n\_steps)$
    \State $noise[idx_1:idx_2] \gets y$
    \State $y\_ori \gets class_0[0]$
    \State $y\_ori[idx_1: idx_2] \gets 0$
    \State $class\_noise[i, :] \gets noise + y\_ori$
\EndFor
\State $y^{class_0} \gets \text{repeat}(class\_noise, \text{repeats}=n, \text{axis}=0)$
\label{alg:toy_test_dataset}
\end{algorithmic}
\end{algorithm}

The data for both classes is combined into a single dataset and reshaped to include a singleton dimension, resulting in an array of shapes (1050, 100, 1). Labels are created as an array of 1000 zeros for Class 0 and 50 ones for Class 1.

\section{Appendix D: Implementation Details}
\label{sec:implementation}

\subsection{Computational Setting}
All the models were implemented in Python 3.9 and realized in PyTorch. All experiments were conducted using a device equipped with an Apple M2 chip featuring an 8-core CPU.

\subsection{Key Hyperparameters}
All hyperparameters can be seen in our codes in supplemental materials. In this subsection, we summarize some key hyperparameters, including the RNN input window length ($N_1$), look-ahead horizons $M$, learning rate $lr$, and penalty term $\lambda$. They generally vary based on different datasets. For all of our test cases, we present them in the following Table \ref{table:hyper-para}.

\begin{table}[h]
\caption{Key hyper-parameters
}
\label{table:hyper-para}
\begin{center}
\begin{small}
\begin{sc}
\resizebox{\columnwidth}{!}{
\begin{tabular}{l|cccccccccccc}
\toprule
&HAR & Earth   &  ECG&  Car   & WorSyn. & Trace & Plane & Fish & Symbol & SynCon. & PV\\
\midrule
$N_1$      & 4  & 6    & 4 & 8  &  12 & 10 & 12 & 14 & 10 & 12 & 10 \\
$M$ & 5 & 10   & 5 & 6 & 10 & 8 & 12 & 14 & 12 & 14 & 4\\
$lr$        &0.001  & 0.0001   & 0.0008 & 0.003   & 0.005 & 0.001 & 0.003 & 0.003 & 0.003 & 0.002 & 0.0002\\
$\lambda$         & 0.2  & 0.01        & 0.1 & 0.01   & 0.01 & 0.01 & 0.005 & 0.02 & 0.03 & 0.01 &0.005\\
\bottomrule
\end{tabular}
}
\end{sc}
\end{small}
\end{center}
\end{table}

\subsection{Code Organization}
Our model contains two parts. Custom RNN Module and ODE-RNN Module.
The Custom RNN Module includes three linear layers for processing the input and hidden states, followed by three output layers. Additionally, three classifier layers are included. The model uses ReLU and Tanh activation functions, with a Sigmoid function in the final layer for classification.
In the forward pass, the model concatenates the input and hidden state tensors, processes them through the RNN layers with ReLU activations, and then through the output and classifier layers. The final output includes the reshaped tensor $u$, the updated hidden state, and the classification output. We use an MLP to parameterize an ODE function for the ODE-RNN model. It consists of two linear layers and a Tanh activation function. The forward pass transforms the input tensor through these layers and applies the Tanh activation, producing the evolved hidden state.
The ODE-RNN model combines the above ODE function with an RNN cell to process sequential data with continuous-time dynamics. It includes an ODE function, a GRU cell, and two linear layers. ReLU activations are used, with a final Sigmoid activation for classification.

The training process involved optimizing the Custom RNN and ODE-RNN models over $400$ epochs using the Adamax optimizer. Each epoch consisted of iterating through batches of data with a batch size of $32$. For each time step, the RNN model processed the input window to produce an intermediate output and hidden state. This intermediate output was then passed to the ODE-RNN model, which evolved the hidden state using the ODE function and updated it through the GRU cell. The models' outputs were compared to the true labels using the CrossEntropyLoss criterion, with regularization terms applied to prevent overfitting.

\section{Appendix E: Additional Visualization for Time-Series Regression}
\label{sec:additional_inter}
In this subsection, we present additional visualization for time-series regression, shown in Fig. \ref{fig:NPC1} to \ref{fig:ODE-RNN3}. For all scenarios, NPC outperforms ODE-RNN for interpolation tasks.

\begin{figure}[h]
\centering
\includegraphics[width=3.4in,trim=5 5 5 0, clip]{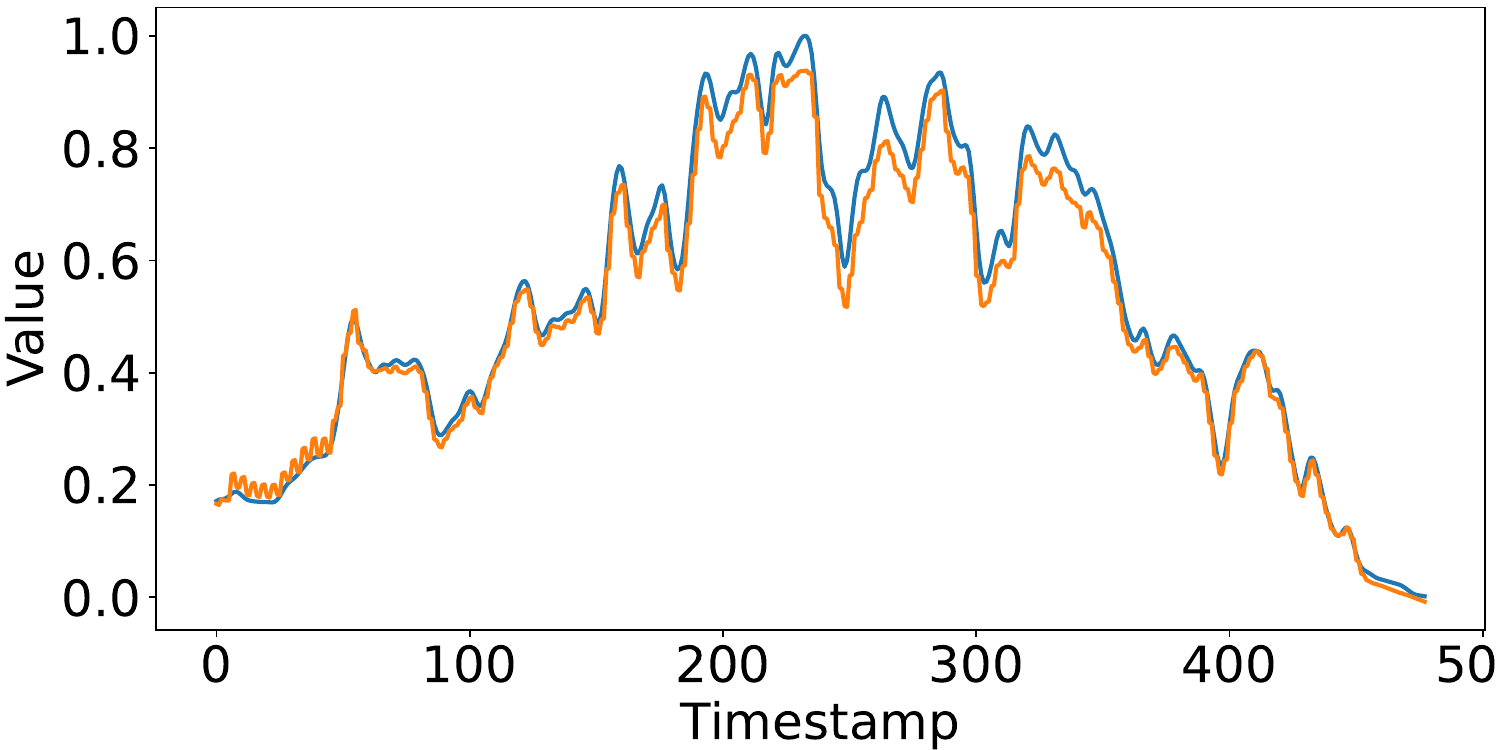}
\centering
\caption{NPC's interpolation result with a data drop rate of $40\%$.}
\label{fig:NPC1}
\end{figure}

\begin{figure}[h!]
\centering
\includegraphics[width=3.4in,trim=5 5 5 0, clip]{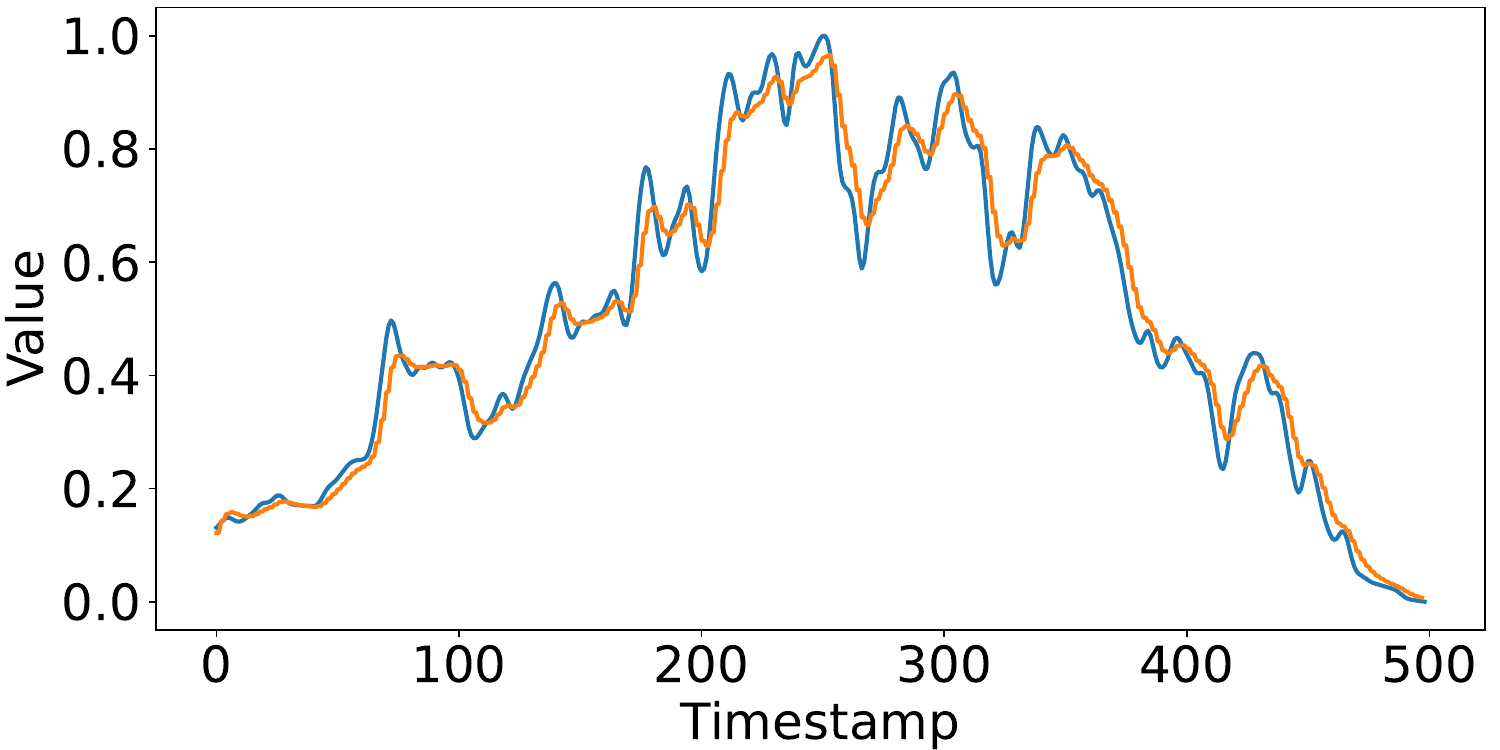}
\centering
\caption{ODE-RNN's interpolation result with a data drop rate of $40\%$.}
\label{fig:ODE-RNN1}
\end{figure}

\begin{figure}[h!]
\centering
\includegraphics[width=3.4in,trim=5 5 5 0, clip]{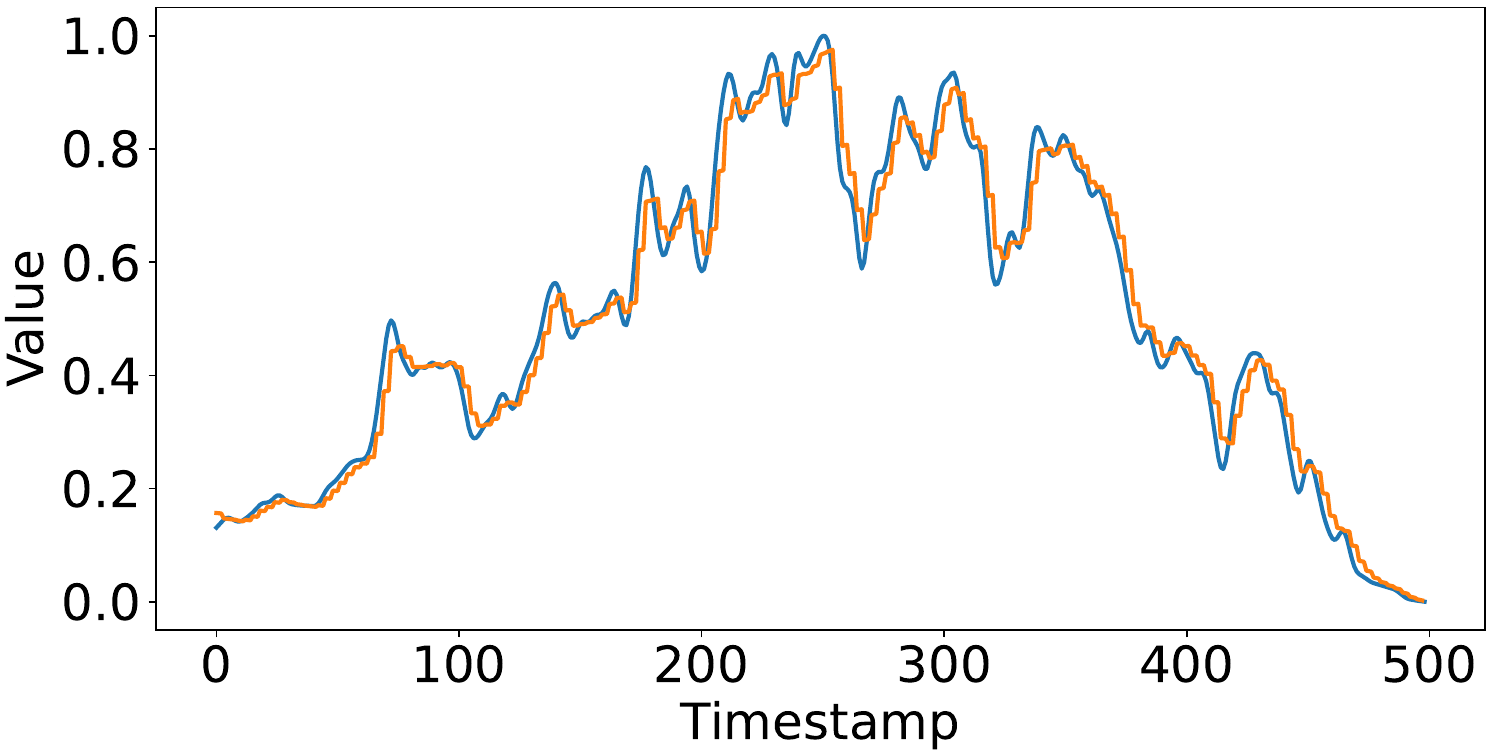}
\centering
\caption{NPC's interpolation result with a data drop rate of $60\%$.}
\label{fig:NPC2}
\end{figure}

\begin{figure}[h!]
\centering
\includegraphics[width=3.4in,trim=5 5 5 0, clip]{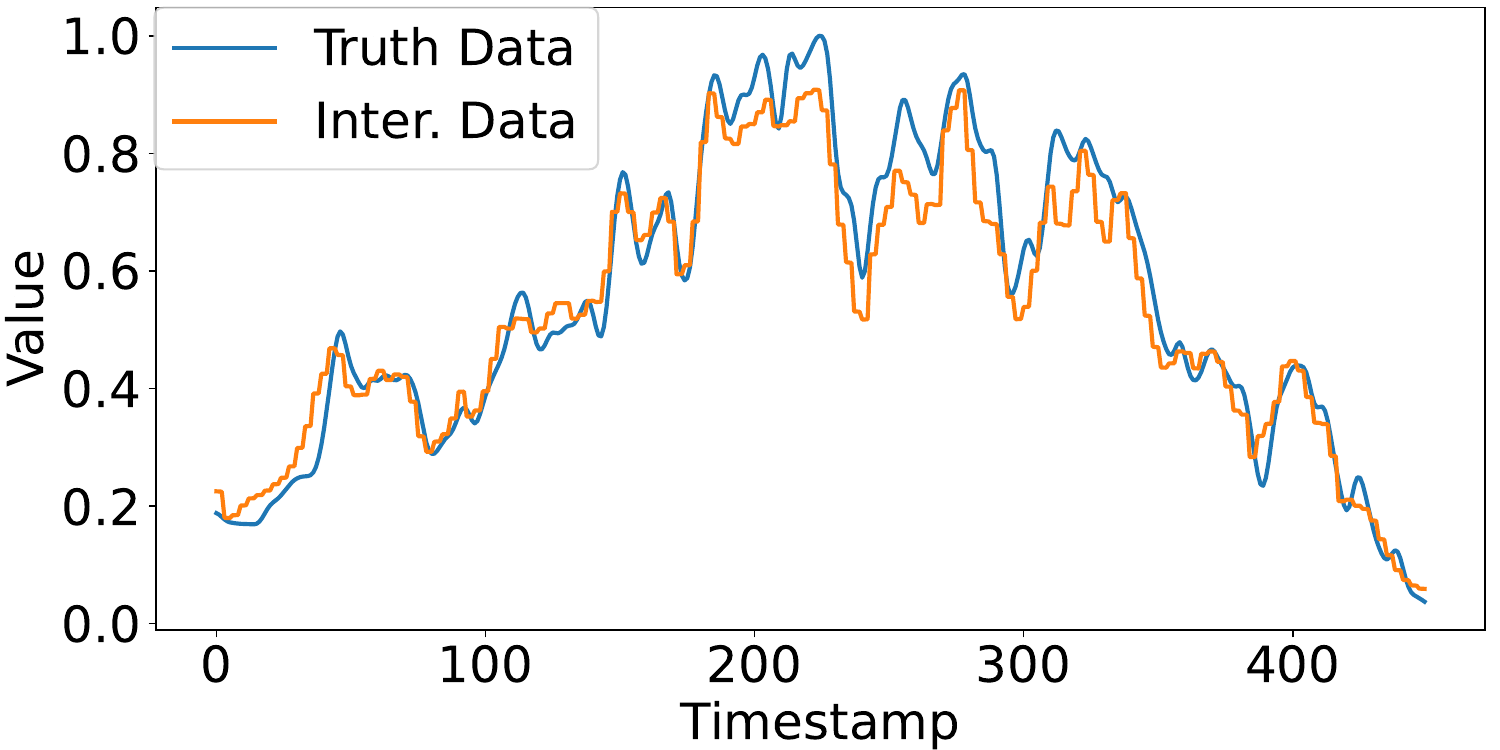}
\centering
\caption{ODE-RNN's interpolation result with a data drop rate of $60\%$.}
\label{fig:ODE-RNN2}
\end{figure}

\begin{figure}[h!]
\centering
\includegraphics[width=3.4in,trim=5 5 5 0, clip]{fig_NPC_interpolation5.pdf}
\centering
\caption{NPC's interpolation result with a data drop rate of $80\%$.}
\label{fig:NPC3}
\end{figure}

\begin{figure}[h!]
\centering
\includegraphics[width=3.4in,trim=5 5 5 0, clip]{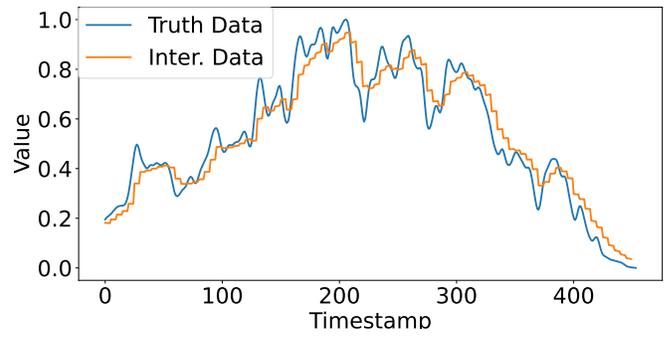}
\centering
\caption{ODE-RNN's interpolation result with a data drop rate of $80\%$.}
\label{fig:ODE-RNN3}
\end{figure}

\end{document}